%% file: cas-dc-template.tex
\documentclass[a4paper]{cas-dc}

\usepackage[authoryear,longnamesfirst]{natbib}
\usepackage{graphicx}
\usepackage{subcaption}
\usepackage{algorithm}
\usepackage{algpseudocode}
\usepackage{enumitem}
\usepackage{multirow}
\usepackage{tabularx}
\usepackage{threeparttable}
\usepackage{float}
\usepackage{placeins}
\usepackage{color, xspace}
\usepackage{appendix}
\usepackage[table]{xcolor}
\usepackage{tcolorbox}
\usepackage{placeins}
\usepackage{xcolor}
\usepackage{booktabs}     
\usepackage{multirow}     
\usepackage{amsmath}      
\usepackage{array}        
\usepackage{caption}      
\usepackage{balance}
\usepackage{booktabs}
\usepackage{siunitx}
\usepackage{makecell}
\usepackage{amsthm}
\usepackage{amsmath,amssymb,amsfonts}

\theoremstyle{plain}        
\newtheorem{theorem}{Theorem}
\newtheorem{lemma}{Lemma}

\theoremstyle{definition}    

\def\tsc#1{\csdef{#1}{\textsc{\lowercase{#1}}\xspace}}
\tsc{WGM}
\tsc{QE}
\tsc{EP}
\tsc{PMS}
\tsc{BEC}
\tsc{DE}

\begin{document}
\let\WriteBookmarks\relax
\def\floatpagepagefraction{1}
\def\textpagefraction{.001}

\shorttitle{TensorSSM: Memorizing Spatial Position in Dynamic State Space}

\shortauthors{Sheng Zhang et~al.}

\title [mode = title]{TensorSSM: Memorizing Spatial Position in Dynamic State Space for Tensor Core Generation}                      



%
\author[1]{Sheng Zhang}[type=editor,
                        orcid=0009-0006-1758-6708]

\fnmark[1]

\ead{szhang844-c@my.cityu.edu.hk}

\ead[url]{https://szhang-cityu.github.io/}

\credit{Conceptualization of this study, Methodology, Experiment, Writing}

\affiliation[1]{organization={City University of Hong Kong},
    addressline={Tat Chee Avenue 83}, 
    city={Hong Kong},
    postcode={999077}, 
    country={China}}

\author[2]{Xinyi Wang}[type=editor,
                        orcid=0009-0009-1741-5310]

\fnmark[1]

\ead{xw3128@columbia.edu}

\credit{Methodology, Experiment, Writing}

\affiliation[2]{organization={Columbia University},
    addressline={116th St \& Broadway}, 
    city={New York},
    postcode={10027}, 
    state={NY},
    country={USA}}

\author[1]{Maolin Wang}[type=editor,
                        orcid=0000-0002-0073-0172]

\cormark[1]

\ead{morin.wang@my.cityu.edu.hk}

\ead[url]{https://morin.wang/}

\author[1]{Xiangyu Zhao}[type=editor,
                        orcid=0000-0003-2926-4416]

\cormark[1]

\ead{xianzhao@cityu.edu.hk}

\ead[url]{https://zhaoxyai.github.io/}

\credit{Conceptualization of this study, Methodology, Supervision}



\cortext[cor1]{Corresponding authors}

\fntext[fn1]{These authors contributed equally to this work.}


\begin{abstract}
High-dimensional tensors are ubiquitous in real-world applications but lead to massive storage overhead. Functional tensor decomposition addresses this by employing neural networks to map element coordinates to tensor cores, significantly reducing the parameters for data compression. However, existing functional methods typically flatten high-order tensors into 1D sequences. This folding operation causes highly overlapping local coordinate sequences and struggles to capture complex cross-mode spatial dependencies, severely limiting their expressiveness on diverse datasets. To tackle these challenges, we propose TensorSSM, a novel Implicit Neural Representation (INR) for functional decomposition method that memorizes spatial positions within a dynamic State Space Model (SSM). Specifically, TensorSSM first decomposes mode-wise coordinates into sequences split by K-ary rule, encoding spatial tokens from coarse to fine, which effectively captures spatial variations. Furthermore, we extend the SSM's state space to explicitly map cross-mode spatial dependencies, allowing inter-mode interactions to enhance the updates of each mode. By integrating cross-mode information from previous hidden states into the current state transitions, TensorSSM effectively expands tensor representation along two dimensions: the sequence dimension for coarse-to-fine spatial tokens, and the state dimension for cross-mode memorization. This design effectively integrates spatial information into the state space and the tensor core, substantially advancing the state adaptability and enhancing the expressiveness of the data compression. Extensive experiments demonstrate that TensorSSM achieves significantly higher reconstruction accuracy and storage efficiency, exhibiting robust expressiveness across diverse tensor datasets without relying on strict spatial assumptions. Our code is available online at \url{https://github.com/szhang-cityu/TensorSSM} for reproducibility.
\end{abstract}



\maketitle

\input{Content/1Introduction}
\input{Content/2Preliminaries}

\input{Content/3Method}
\input{Content/4Experiment}
\input{Content/5Conclusion}
\bibliographystyle{cas-model2-names}

\bibliography{cas-refs}

\input{Content/6Appendix}





\end{document}

%% file: Content/1Introduction.tex
\section{Introduction}
Tensors, defined as high-dimensional arrays, serve as a ubiquitous and fundamental mathematical structure for modeling complex, multi-modal data in the real world. A vast range of real-world datasets can be naturally formulated as tensors, encompassing diverse domains such as remote sensing images~\citep{huang2020nonlocal, shu2023dual, shu2026physics}, urban traffic networks~\citep{zhu2023difftraj}, financial stock markets~\citep{kwon2023tensorcodec, jang2021fast}, and user behavioral data~\citep{zhang2025glint, wang2025star, wang2025findrec}. However, as the dimensionality and resolution of modern data continue to scale, their corresponding tensors inherently contain a massive number of entries. This exponential growth in data size leads to storage overhead and computational burdens, preventing efficient deployment in resource-constrained environments~\citep{kwon2023tensorcodec, zhang2024dns}. Consequently, developing compact representations that can precisely express and reconstruct these high-dimensional tensors using fewer parameters has emerged as an increasingly important research problem.

\begin{figure*}[t]
    \centering
    \includegraphics[width=0.97\linewidth]{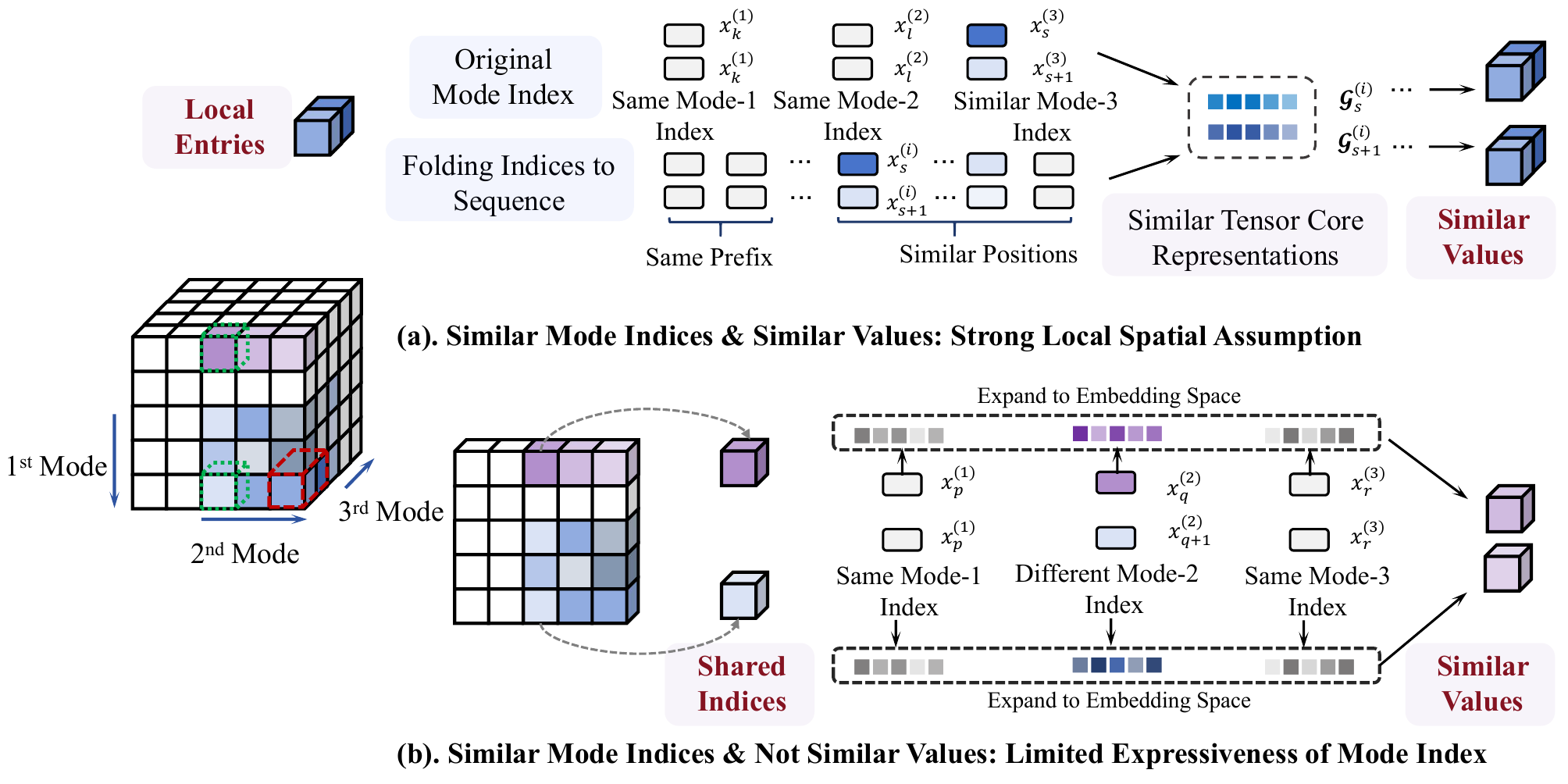}
    \vspace{-3mm}
    \caption{(a). Adjacent tensor entries usually have highly overlapping raw mode indices or sequences from folded mode indices, leading to the same or similar tensor core representations, which is only effective for local similar entries. (b). Entries with shared mode indices may have similar representations, but their true values may be totally different from each other.}
    \label{fig:challenges}
    \vspace{-5mm}
\end{figure*}

To tackle this challenge, conventional matrix and tensor decomposition techniques have been extensively developed to reconstruct the original high-dimensional data using a substantially reduced set of core elements~\citep{kwon2023tensorcodec, yin2021tt}. By factorizing the original tensor into lower-dimensional components, these methods achieve significant data compression and parameter efficiency~\citep{lebedev2014speeding}. Despite their widespread application, however, these traditional decomposition approaches often suffer from inherent limitations. For instance, certain methods are strictly confined to matrix-level factorization, which are not suitable for multi-dimensional data analysis~\citep{sun2007less}. Furthermore, many existing techniques heavily rely on strong structural assumptions, such as the sparsity of the data~\citep{allen2012sparse, li2017model, morup2008algorithms, li2020sgd} and the low-rank properties~\citep{malik2018low, zhao2019learning, gong2024enhanced}. This over-reliance on specific priors severely restricts their applicability and expressiveness when dealing with dense, complex, or irregular real-world tensors.

To mitigate the reconstruction errors inherent in high-order tensors and enhance the expressiveness for a broader range of tensor-based data, functional tensor decomposition methods~\citep{vemuri2024functional, zhou2026efficient, wang2025functional, zhang2026online} have been proposed to learn high-quality INR as tensor cores. Departing from conventional decomposition paradigms, these approaches leverage neural networks to capture the intrinsic relational patterns among tensor entries. Specifically, the multi-dimensional coordinate indices of the tensor elements are mapped into tensor cores via learned transformation rules~\citep{kwon2023tensorcodec}, from which the original numerical values are subsequently recovered through tensor decomposition operations. By implicitly representing the entire tensor solely through the neural network, this paradigm achieves exceptional parameter efficiency, requiring minimal storage for model weights while simultaneously maintaining remarkably low reconstruction errors.

Despite the great success of these functional approaches, they inherently struggle with several critical limitations when applied to complex tensors with high dimensionality in each mode. Specifically, to process these data, existing INR methods directly read the coordinates as the input sequence to learn the spatial correlation~\citep{zhang2026online, fang2024functional} or employ folding strategy~\citep{kwon2023tensorcodec} to decompose coordinate sequence to reduce the parameter size and refine the learning of spatial positions of tensor elements. 
However, this paradigm introduces a major bottleneck.
Existing functions learn tensor cores from raw index sequences, which \textbf{struggle to balance the \textit{reliance on strong assumption of high local spatial correlation} and the \textit{expressiveness of each index}}. 
Specifically, on the one hand, the folding operation inevitably generates highly overlapping coordinate sequences for local entries within the high-order tensor as shown in Figure~\ref{fig:challenges}.(a). Consequently, these methods heavily rely on the strong assumption of high local spatial correlation. To mitigate the disruption caused by arbitrary folding, techniques such as the reordering operation~\citep{kwon2023tensorcodec} are required to group correlated entries. Nevertheless, the effectiveness of such reordering heuristics is inherently limited and often fails to generalize across diverse and heterogeneous datasets. 
on the other hand, when the dimension of each mode is large, there are many entries that share the same indices in different modes as shown in Figure~\ref{fig:challenges}.(b), which limits the expressiveness of the index representations in the neural networks. In addition, storing the explosive spatial correlation in the limited hidden states makes it exceedingly difficult for the neural network to effectively capture and integrate them. 

To address these challenges, we propose TensorSSM, a novel Implicit Neural Representation (INR) that memorizes the spatial states of tensor elements into the state space of State Space Models (SSMs)~\citep{gu2021efficiently, gu2023mamba}. Specifically, instead of learning from raw index sequences, we first decompose the coordinate value of each tensor entry along every mode into a spatial token sequence based on a K-ary numeral system. This hierarchical decomposition encodes the spatial relationships and variations within each mode in a coarse-to-fine manner, where higher-order digits govern broader spatial regions and lower-order digits capture fine-grained local correlations. 
Furthermore, since standard sequential features can only capture intra-mode numerical relationships and suffer from limited expressiveness, we inject the cross-mode spatial dependencies into the state transition in the state space by spreading the spatial information in the state transition operator, yielding a spatial-aware SSM for tensors. Through these designs, TensorSSM shifts the learning paradigm from raw index sequences to expanded index sequences for state memorization, extending the tensor learning into two dimensions: (1) \textbf{the sequence dimension along the entry indices}, which controls sequence variations to learn coarse-to-fine spatial discrepancies; and (2) \textbf{the state dimension}, which compresses and memorizes cross-mode spatial information to govern numerical updates.

Our main contributions are summarized as follows:

\begin{itemize}[leftmargin=*]
\item We propose TensorSSM, a novel INR that seamlessly integrates spatial information learning for large tensors. By expanding indices into a coarse-to-fine spatial token via K-ary decomposition, our method effectively controls spatial discrepancies and overcomes the expressiveness limitations in raw index representations.
\item We introduce a state space extension that compresses and memorizes spatial information from other modes into the shared state, while simultaneously controlling hidden state updates with cross-mode information, which significantly enhances the model's expressiveness based on complex spatial correlations.
\item Extensive experiments on diverse real-world tensor data-sets demonstrate that TensorSSM achieves superior reconstruction accuracy and parameter efficiency compared to state-of-the-art baselines, validating its robustness and generalization capabilities across heterogeneous data.
\end{itemize}

%% file: Content/2Preliminaries.tex
\section{Preliminaries}
\label{sec:prelim}

\subsection{Functional Tensor Decomposition}
\label{subsec:tensor-function}

\noindent\textbf{Tensor Decomposition.}
Given a tensor $\mathcal{X} \in \mathbb{R}^{N_1 \times \cdots \times N_d}$, where each mode $k$ contains $N_k$ objects, and each entry of $\mathcal{X}$ is located by a $d$-size tuple of indexes $(i_1, \ldots, i_d)$, denoted as $\mathcal{X}(i_1, \cdots, i_d)$. Tensor decomposition approximates each entry by a compact interaction of low-rank latent factors. Specifically, we introduce $d$ groups of latent factors $\{\boldsymbol{U}_1, \ldots, \boldsymbol{U}_d\}$, where $\boldsymbol{U}_k = [\boldsymbol{u}^k_1, \ldots, \boldsymbol{u}^k_{N_k}]^\top \in \mathbb{R}^{N_k \times r_k}$, and each factor $\boldsymbol{u}^k_{i_k} \in \mathbb{R}^{r_k}$ represents the latent embedding of the $i_k$-th object in the $k$-th mode.

The classic Canonic-Polyadic (CP)  decomposition~\citep{hitchcock1927expression} sets all modes' ranks to a single value $r$, and approximates each entry as the product-sum of the factors: $ \mathcal{X}(i_1, \cdots, i_d) \approx \big(\boldsymbol{u}^1_{i_1} \circ \boldsymbol{u}^2_{i_2} \circ \cdots \circ \boldsymbol{u}^d_{i_d}\big)^\top \boldsymbol{1}$, where $\circ$ denotes the element-wise (Hadamard) product. Tucker decomposition~\citep{tucker1966some} takes one more step towards a more expressive and flexible low-rank structure. It allows us to set different factor ranks $\{r_1, \ldots, r_d\}$ for each mode, by introducing a core tensor $\mathcal{W} \in \mathbb{R}^{r_1 \times \cdots \times r_d}$, known as the Tucker core. It models the entry as a weighted sum over all possible cross-rank factor multiplications: $ \mathcal{X}(i_1, \cdots, i_d) \approx \mathrm{vec}(\mathcal{W})^\top \big(\boldsymbol{u}^1_{i_1} \otimes \boldsymbol{u}^2_{i_2} \otimes \cdots \otimes \boldsymbol{u}^d_{i_d}\big)$, where $\otimes$ is the Kronecker product and $\mathrm{vec}(\cdot)$ is the vectorization operator. Tensor-train (TT) decomposition~\citep{oseledets2011tensor} represents the tensor by $d$ third-order tensors $\mathcal{G}_1, \ldots, \mathcal{G}_d$, called TT cores, so that each entry is approximated as a chain of the core slices: $\mathcal{X}(i_1, \cdots, i_d) \approx \mathcal{G}_1^{(2)}(i_1)\mathcal{G}_2^{(2)}(i_2)\cdots\mathcal{G}_d^{(2)}(i_d), $
where, for all $k$, $\mathcal{G}_k \in \mathbb{R}^{r_{k-1} \times N_k \times r_k}$, $\mathcal{G}_k^{(2)}(i_k) \in \mathbb{R}^{r_{k-1} \times r_k}$ is the $i_k$-th slice of $\mathcal{G}_k$ along the second mode, and $r_0=r_d=1$. For simplicity, in this paper, we unify all other TT ranks (i.e., $r_1, \cdots, r_{d-1}$) to a single value denoted by $r$.

\noindent\textbf{Function factorization.}
Function factorization refers to decomposing a complex multivariate function into a set of low-dimensional functions~\citep{rai2014sparse,nouy2015low}. One of the most widely used methods for function factorization is TT~\citep{gorodetsky2015function,bigoni2016spectral, ballester2019sobol}, which converts the function approximation into a tensor decomposition problem. TT decomposition can factorize a $d$-variate function $f(x_1, \ldots, x_d)$ by:
\begin{equation}
  f(x_1, \cdots, x_d) \approx \mathcal{G}_1(x_1)\mathcal{G}_2(x_2)\cdots\mathcal{G}_d(x_d),
  \label{eq:func-tt}
\end{equation}
where each function $\mathcal{G}_k(\cdot): \mathbb{R} \to \mathbb{R}^{r_{k-1} \times r_k}$ is a univariate matrix-valued function, which takes the scalar $x_k$ as the input and outputs a matrix. Similarly, Tucker cores can also be used to generalize the discrete factors:
\begin{equation}
  f(x_1, \cdots, x_d) \approx \mathrm{vec}(\mathcal{W})^\top 
  \big(\boldsymbol{u}^1(x_1) \otimes \boldsymbol{u}^2(x_2) \otimes \cdots \otimes \boldsymbol{u}^d(x_d)\big),
  \label{eq:func-tucker}
\end{equation}
where each $\boldsymbol{u}^k(\cdot): \mathbb{R} \to \mathbb{R}^{r_k}$ maps the continuous index $x_k$ of mode $k$ to an $r_k$-dimensional latent factor, extending Tucker decomposition from grid-structured data to continuous-indexed data.

\subsection{Structured State Space Model (S4)}
\label{subsec:s4}

\noindent\textbf{State space model (SSM).}
A linear state space model maps a 1-D input signal $u(t) \in \mathbb{R}$ to an output 
$y(t) \in \mathbb{R}$ through an $N$-dimensional latent state $\boldsymbol{h}(t)$:
\begin{equation}
  \boldsymbol{h}'(t) = \boldsymbol{A}\boldsymbol{h}(t) + \boldsymbol{B}u(t),
  \qquad
  y(t) = \boldsymbol{C}\boldsymbol{h}(t),
  \label{eq:continuous-ssm}
\end{equation}
where $\boldsymbol{A} \in \mathbb{R}^{N \times N}$, $\boldsymbol{B} \in \mathbb{R}^{N \times 1}$, and $\boldsymbol{C} \in \mathbb{R}^{1 \times N}$. To operate on a discrete sequence $\boldsymbol{u} = (u_0, \ldots, u_{L-1})$, Eq.~\eqref{eq:continuous-ssm} is discretized by a step size $\Delta$ via the bilinear transform, yielding the recurrent inference process:
\begin{equation}
  \boldsymbol{h}_k = \bar{\boldsymbol{A}}\,\boldsymbol{h}_{k-1} + \bar{\boldsymbol{B}}\,u_k,
  \qquad
  y_k = \boldsymbol{C}\,\boldsymbol{h}_k,
  \label{eq:recurrent-ssm}
\end{equation}
where $\bar{\boldsymbol{A}} = (\boldsymbol{I} - \tfrac{\Delta}{2}\boldsymbol{A})^{-1}(\boldsymbol{I} + \tfrac{\Delta}{2}\boldsymbol{A})$ and $\bar{\boldsymbol{B}} = (\boldsymbol{I} - \tfrac{\Delta}{2}\boldsymbol{A})^{-1}\Delta\boldsymbol{B}$. Unrolling Eq.~\eqref{eq:recurrent-ssm} with $\boldsymbol{h}_{-1} = \boldsymbol{0}$ gives an equivalent global convolution
\begin{equation}
  \boldsymbol{y} = \bar{\boldsymbol{K}} * \boldsymbol{u},
  \qquad
  \bar{\boldsymbol{K}} = \big(\boldsymbol{C}\bar{\boldsymbol{B}},\; 
  \boldsymbol{C}\bar{\boldsymbol{A}}\bar{\boldsymbol{B}},\; \ldots,\; 
  \boldsymbol{C}\bar{\boldsymbol{A}}^{L-1}\bar{\boldsymbol{B}}\big),
  \label{eq:conv-ssm}
\end{equation}
where $\bar{\boldsymbol{K}}$ is the SSM convolution kernel. The model supports parallel training through~\eqref{eq:conv-ssm} and linear-time recurrent inference through~\eqref{eq:recurrent-ssm}. However, materializing $\bar{\boldsymbol{K}}$ naively costs $O(N^2 L)$ time and $O(NL)$ space, and a generic $\boldsymbol{A}$ fails to capture long-range dependencies.

\noindent\textbf{The S4 parameterization.}
S4~\citep{gu2021efficiently} addresses both problems by imposing structure on $\boldsymbol{A}$. First, $\boldsymbol{A}$ is initialized by the HiPPO matrix~\citep{gu2022train, gu2020hippo}, which makes $\boldsymbol{h}(t)$ optimally memorize the history of $u(t)$. Second, S4 exploits the key structural result that every HiPPO matrix admits a normal plus low-rank (NPLR) decomposition
\begin{equation}
  \boldsymbol{A} = \boldsymbol{V}\big(\boldsymbol{\Lambda} - \boldsymbol{P}\boldsymbol{Q}^*\big)\boldsymbol{V}^*,
  \label{eq:nplr}
\end{equation}
where $\boldsymbol{V}$ is unitary, $\boldsymbol{\Lambda}$ diagonal, and $\boldsymbol{P}, \boldsymbol{Q}$ are 
low-rank factors. 
Then, the kernel $\bar{\boldsymbol{K}}$ is computed stably in $\tilde{O}(N + L)$ operations and $O(N + L)$ space, while $(\boldsymbol{\Lambda}, \boldsymbol{P}, \boldsymbol{Q}, \boldsymbol{B}, \boldsymbol{C}, \Delta)$ are the learnable parameters.

%% file: Content/3Method.tex
\section{Method}

In this section, we present the proposed TensorSSM for learning a spatial-aware INR over arbitrary tensor indices. Given an index \(\boldsymbol{i} \in \mathbb{N}^M\) of tensor $\boldsymbol{\mathcal{X}}$, where \(M\) is the number of tensor modes, the model outputs a scalar prediction \(\hat{y}\). The method consists of three main components: (i) coarse-to-fine coordinate tokenization (Section~\ref{tokenization}), (ii) a spatial-aware TensorSSM with cross-stream fusion and input-dependent rank-one transitions (Section~\ref{TensorSSM}), and (iii) Tucker-based reconstruction with pairwise residuals (Section~\ref{sec:method-reconstruction}).

\begin{figure*}[t]
    \centering
    \includegraphics[width=\linewidth]{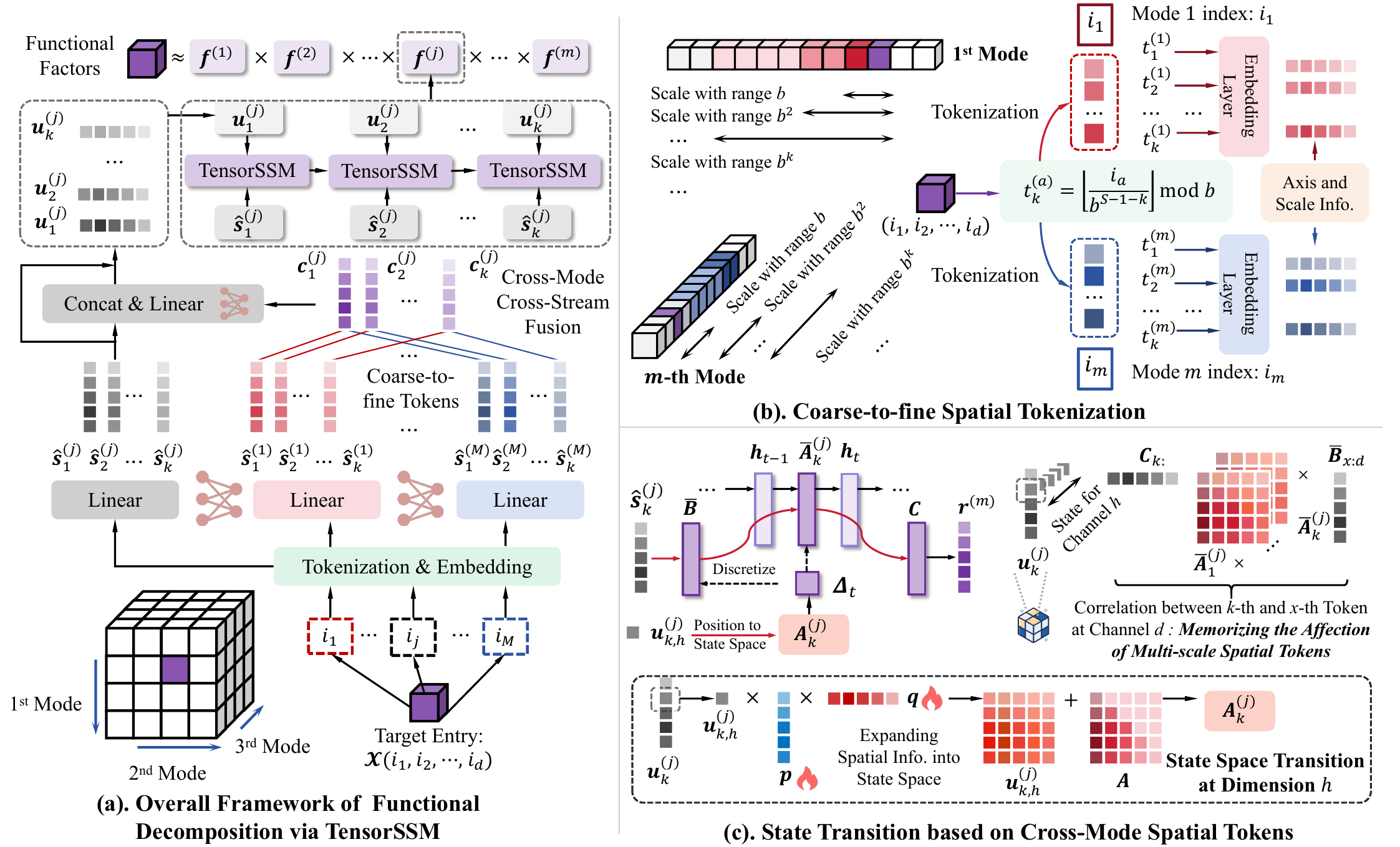}
    \vspace{-1mm}
    \caption{(a). Framework Overview of TensorSSM via Functional Tensor Decomposition. (b). The index of each mode is decomposed via base-$b$ rule and form a coarse-to-fine sequence as the spatial token. (c). The cross-mode information is spanned to the state transition operator, selectively influencing the spatial dependencies in the tensor core representations.}
    \label{fig:framework}
    \vspace{-5mm}
\end{figure*}

\subsection{Coarse-to-Fine Coordinate Tokenization}
\label{tokenization}
Existing INR methods leverage dense indices to generate tensor cores, limiting the expressiveness of the representations and rely heavily on the local spatial assumptions~\citep{kwon2023tensorcodec}. To handle tensors of arbitrary order and varying mode sizes, we avoid direct dense representations and instead decompose each multi-dimensional index $\boldsymbol{i} \in \mathbb{N}^M$ into a multi-scale sequence of tokens. 

Instead of representing a coordinate densely, we expand each coordinate in a base-\(b\) numeral system. We determine the number of scales $S$ required to represent the largest mode in a base-$b$ numeral system:
\begin{equation}
    S = \left\lceil \log_b \left( \max_{m} N_m \right) \right\rceil.
\end{equation}
Then, for each mode $m$, the coordinate $i_m$ of the index \(\boldsymbol{i} = (i_1, \dots, i_M)\) is expanded into $S$ base-$b$ digits, producing a coarse-to-fine sequence of spatial tokens:
\begin{equation}
    t_k^{(m)}
    =
    \left\lfloor
    \frac{i_{m}}{b^{S-1-k}}
    \right\rfloor
    \bmod b, \quad k \in \{0,\dots,S-1\}
\end{equation}
By the digit-wise variation structure of base-$b$ expansions, the $k$-th position of the token stores the spatial information in the range of $b^k$. Therefore, these tokens serve as media for storing multi-scale spatial information in dynamic states.

We project these discrete identifiers into a continuous latent space $\mathbb{R}^H$ via a disentangled additive composition of learnable embeddings to enhance the expressiveness of the model~\citep{wang2026embedding}. Formally, we define three embedding matrices: the token embedding $\boldsymbol{W}^{\mathrm{tok}} \in \mathbb{R}^{b \times H}$, the axis embedding $\boldsymbol{W}^{\mathrm{ax}} \in \mathbb{R}^{M \times H}$, and the scale embedding $\boldsymbol{W}^{\mathrm{sc}} \in \mathbb{R}^{S \times H}$. The embedding operation acts as a row-selection operator indexed by the integer identifiers. Specifically, for an integer index $j$ and matrix $\boldsymbol{W} \in \mathbb{R}^{J \times H}$, the notation $\boldsymbol{W}_j$ denotes the $j$-th row of $\boldsymbol{W}$, yielding a vector in $\mathbb{R}^H$. The embedded token sequence $\boldsymbol{s}_k^{(d)} \in \mathbb{R}^H$ is then constructed as:
\begin{equation}
    \boldsymbol{s}_k^{(m)} = \boldsymbol{W}^{\mathrm{tok}}_{t_k^{(m)}} + \boldsymbol{W}^{\mathrm{ax}}_{m} + \boldsymbol{W}^{\mathrm{sc}}_{k}.
\end{equation}

This embedding strategy preserves structural semantics, with each embedded token carrying the features of the corresponding tensor entry. These features can be expanded and used to memorize the state of cross-mode correlation patterns among tensor entries in our TensorSSM.

\subsection{TensorSSM: An INR with Memory}
\label{TensorSSM}
Memory has become an increasingly functions for both machine learning models~\citep{lin2026self} and large language model agents~\citep{zhang2026memsearch}. Existing INR methods via functional tensor decomposition employ hidden states in neural networks to learn the tensor core representations~\citep{vemuri2026f}. However, the vectors of hidden states show limited expressiveness, and struggle to memorize the spatial dependencies~\citep{gao2019representation}. In addition, basic SSMs are developed to learn sequential dependencies instead of spatial memory~\citep{gu2020hippo, sarrof2024expressive}. Uniquely, unlike a standard 1D SSM, our model explicitly spread the influence of multiple tensor modes into the state transition operator with information exchanged and memorized across modes through synchronized fusion and conditional modulation. 

\subsubsection{Cross-Mode Cross-Stream Fusion}

While each mode is tokenized independently, modeling an INR that can memorize spatial positions for the tensor requires interactions among different modes. To enable this without disturbing the coarse-to-fine causal order inside each stream, we apply a Synchronous Cross-Stream Fusion operation at the $k$-th position of the token. Let
\begin{equation}
    \widehat{\boldsymbol{s}}_k^{(m)}
    =
    \operatorname{LN}\!\left(\boldsymbol{s}_k^{(m)}\right)
\end{equation}
denote the token embedding of mode \(m\) at scale \(k\) after the layer normalization. The cross-mode context for mode \(m\) is aggregated from all other modes at the same scale using a message aggregation function $\operatorname{MSG}(\cdot)$:
\begin{equation}
    \begin{aligned}
    \bar{\boldsymbol{c}}_k^{(m)}
    & =
    \frac{1}{M-1}
    \sum_{m' \neq m}
    \operatorname{MSG}
    \left(
        \widehat{\boldsymbol{s}}_k^{(m)},
        \widehat{\boldsymbol{s}}_k^{(m')}
    \right)\\
    & = \frac{1}{M-1}
    \sum_{m' \neq m}
    \boldsymbol{W}_{o}
    \operatorname{SiLU}
    \left(
        \boldsymbol{W}_k^{(m)} \widehat{\boldsymbol{s}}_k^{(m)}
        +
        \boldsymbol{W}_k^{(m')} \widehat{\boldsymbol{s}}_k^{(m')}
    \right),
\end{aligned}
\end{equation}
where $\boldsymbol{W}_{o},  \boldsymbol{W}_k^{(m)}, \boldsymbol{W}_k^{(m')} \in \mathbb{R}^{H \times H}$ are learnable matrices and $H$ is the hidden size of the model. The fused token, which serves as the input to the SSM stream, is given by
\begin{equation}
    \boldsymbol{u}_k^{(m)}
    =
    \boldsymbol{s}_k^{(m)}
    +
         \boldsymbol{W}_{2}
         \operatorname{SiLU}
         \left(
        \boldsymbol{W}_{1}
        \left(
            [
                \widehat{\boldsymbol{s}}_k^{(m)}
                \,\Vert\,
                \bar{\boldsymbol{c}}_k^{(m)}
            ]
        \right)
        \right).
\end{equation}
If $M=1$, we will set $\bar{\boldsymbol{c}}_k^{(1)} = 0$ under this circumstance, as there is no cross-mode information. Crucially, \(\boldsymbol{u}_k^{(m)}\) is an explicit function of all mode embeddings aggregated at the same scale: $\boldsymbol{u}_k^{(m)} = \Phi_k \left( \boldsymbol{s}_k^{(1)}, \boldsymbol{s}_k^{(2)},\dots, \boldsymbol{s}_k^{(M)} \right)$, which is utilized to transfer the spatial information to the transition operator.
Because the fusion uses the original pre-fusion tokens \(\{\boldsymbol{s}_k^{(m)}\}_{m=1}^M\) synchronously, the operation is permutation equivariant with respect to mode indices. This allows the model to incorporate cross-mode spatial information in the multi-scale positional representation, which serves as the spatial correlation that can be memorized by the SSM state, enhancing the expressiveness of the INR.
\subsubsection{Spatial-Aware State Transition for SSM}

Existing methods including functional decomposition methods based on neural networks, struggle to capture the spatial patterns and numerical values simultaneously in the channels due to the limited expressiveness of the vector features. Beyond traditional S4~\citep{gu2021efficiently}, we utilize the state space transition with dynamic spatial information, enabling the SSM to compress and memorize the spatial patterns in the state space. 

Specifically, after cross-stream fusion, the fused token \(\boldsymbol{u}_k^{(m)}\) is fed into the SSM. For clarity, we write the equations for a single feature channel and omit the batch index. In the full model, the equations are applied independently over \(H\) channels. For simplicity, we denote the fused token at some channel as \(\boldsymbol{u}_{k}^{(m)}\). Let the continuous-time SSM be
\begin{equation}
    \frac{d \boldsymbol{x}^{(m)}(t)}{dt}
    =
    \boldsymbol{A}_k^{(m)} \boldsymbol{x}^{(m)}(t)
    +
    \boldsymbol{b}_{\mathrm{ssm}} \boldsymbol{u}_k^{(m)},
    \quad
    y_k^{(m)}
    =
    \boldsymbol{c}_{\mathrm{ssm}}^{\top}
    \boldsymbol{x}^{(m)}(t_k),
\end{equation}
where \(\boldsymbol{x}^{(m)}(t) \in \mathbb{R}^{N}\) is the latent state, \(\boldsymbol{b}_{\mathrm{ssm}} \in \mathbb{R}^{N}\) is the input projection vector, and \(\boldsymbol{c}_{\mathrm{ssm}} \in \mathbb{R}^{N}\) is the readout vector.

The transition matrix is initialized with the HiPPO-LegS~\citep{gu2020hippo} structure and modulated by the cross-mode spatial information tuning so that the correlation among multiple entries can be memorized:
\begin{equation}
    \label{transition}
    \boldsymbol{A}_k^{(m)}
    =
    \boldsymbol{A}
    +
    \alpha \, \boldsymbol{u}_k^{(m)} \, \boldsymbol{p} \boldsymbol{q}^{\top},
\end{equation}
where \(\boldsymbol{A} \in \mathbb{R}^{N \times N}\) is the HiPPO base matrix, \(\boldsymbol{p}, \boldsymbol{q} \in \mathbb{R}^{N}\) are learnable vectors, and \(\alpha > 0\) is a scaling factor. This design introduces input-dependent selectivity while preserving the stable memory structure of HiPPO. Specifically, we use learnable $\boldsymbol{p}$ and $\boldsymbol{q}$ to selectively span the spatial influence to the state update operator $\boldsymbol{A}_k^{(m)}$. Therefore, the multi-scale position information and patterns from other modes can be selectively compressed into the state, enhancing the expressiveness of learning tensor entries. The values are not only learned from indices, but also from the correlation pattern among all the indices of arbitrary entries memorized in the state space, helping the SSM to no longer rely on the strong assumption of high local spatial correlation. Then, following S4, we discretize the transition operator with the continuous spatial information using the bilinear transform with step size \(\Delta t\):
\begin{equation}
    \bar{\boldsymbol{A}}_k^{(m)}
    =
    \left(
        \boldsymbol{I}
        -
        \frac{\Delta t}{2}
        \boldsymbol{A}_k^{(m)}
    \right)^{-1}
    \left(
        \boldsymbol{I}
        +
        \frac{\Delta t}{2}
        \boldsymbol{A}_k^{(m)}
    \right),
\end{equation}
\begin{equation}
    \bar{\boldsymbol{b}}_k^{(m)}
    =
    \left(
        \boldsymbol{I}
        -
        \frac{\Delta t}{2}
        \boldsymbol{A}_k^{(m)}
    \right)^{-1}
    \Delta t \,
    \boldsymbol{b}_{\mathrm{ssm}}.
\end{equation}
The discrete-time recurrence is then
\begin{equation}
    \boldsymbol{x}_k^{(m)}
    =
    \bar{\boldsymbol{A}}_k^{(m)}
    \boldsymbol{x}_{k-1}^{(m)}
    +
    \bar{\boldsymbol{b}}_k^{(m)}
    \boldsymbol{u}_k^{(m)},
    \qquad
    y_k^{(m)}
    =
    \boldsymbol{c}_{\mathrm{ssm}}^{\top}
    \boldsymbol{x}_k^{(m)}.
\end{equation}

Directly forming and inverting $\boldsymbol{I}-\frac{\Delta t}{2}\boldsymbol{A}_k^{(m)}$ at every step would be computationally expensive. However, since the input dependence is rank-one, we can derive an exact and efficient update using the Woodbury identity. Define
\begin{equation}
    \boldsymbol{L}_0
    =
    \boldsymbol{I}
    -
    \frac{\Delta t}{2}
    \boldsymbol{A},
    \quad
    \boldsymbol{R}_0
    =
    \boldsymbol{I}
    +
    \frac{\Delta t}{2}
    \boldsymbol{A},
    \quad
    \boldsymbol{\gamma}_k^{(m)}
    =
    \frac{\Delta t}{2}
    \alpha \,
    \boldsymbol{u}_k^{(m)}.
\end{equation}
Then using the Eq.~\eqref{transition} we have:
\begin{equation}
    \boldsymbol{L}_k^{(m)}
    :=
    \boldsymbol{I}
    -
    \frac{\Delta t}{2}
    \boldsymbol{A}_k^{(m)}
    =
    \boldsymbol{L}_0
    -
    \boldsymbol{\gamma}_k^{(m)}
    \boldsymbol{p} \boldsymbol{q}^{\top}.
\end{equation}
The state update can be written as solving
\begin{equation}
    \boldsymbol{L}_k^{(m)}
    \boldsymbol{x}_k^{(m)}
    =
    \boldsymbol{v}_k^{(m)},
\end{equation}
where
\begin{equation}
    \boldsymbol{v}_k^{(m)}
    =
    \boldsymbol{R}_0
    \boldsymbol{x}_{k-1}^{(m)}
    +
    \boldsymbol{\gamma}_k^{(m)}
    \boldsymbol{p}
    \boldsymbol{q}^{\top}
    \boldsymbol{x}_{k-1}^{(m)}
    +
    \Delta t \,
    \boldsymbol{b}_{\mathrm{ssm}}
    \boldsymbol{u}_k^{(m)}.
\end{equation}

By the Woodbury matrix identity,
\begin{equation}
    \left(
        \boldsymbol{L}_k^{(m)}
    \right)^{-1}
    =
    \boldsymbol{L}_0^{-1}
    +
    \boldsymbol{\gamma}_k^{(m)}
    \boldsymbol{L}_0^{-1}
    \boldsymbol{p}
    \left(
        1
        -
        \boldsymbol{\gamma}_k^{(m)}
        \boldsymbol{q}^{\top}
        \boldsymbol{L}_0^{-1}
        \boldsymbol{p}
    \right)^{-1}
    \boldsymbol{q}^{\top}
    \boldsymbol{L}_0^{-1}.
\end{equation}
Therefore,
\begin{equation}
    \boldsymbol{x}_k^{(m)}
    =
    \boldsymbol{L}_0^{-1}
    \boldsymbol{v}_k^{(m)}
    +
    \frac{
        \boldsymbol{\gamma}_k^{(m)}
        \boldsymbol{q}^{\top}
        \boldsymbol{L}_0^{-1}
        \boldsymbol{v}_k^{(m)}
    }{
        1
        -
        \boldsymbol{\gamma}_k^{(m)}
        \boldsymbol{q}^{\top}
        \boldsymbol{L}_0^{-1}
        \boldsymbol{p}
    }
    \boldsymbol{L}_0^{-1}
    \boldsymbol{p}.
\end{equation}

This formulation is algebraically equivalent to solving the bilinearly discretized dynamic system that memorizes the spatial information using the spatial-aware state transition matrix $\boldsymbol{A}_k^{(m)}$. It avoids calculating matrix inversion at every step. Once $\boldsymbol{L}_0^{-1}$, $\boldsymbol{L}_0^{-1}\boldsymbol{p}$, $\boldsymbol{q}^{\top}\boldsymbol{L}_0^{-1}$, and the scalar $\boldsymbol{q}^{\top}\boldsymbol{L}_0^{-1}\boldsymbol{p}$ are precomputed, each step only requires matrix-vector products and a scalar correction. Moreover, when \(\boldsymbol{u}_k^{(m)} \approx 0\), the transition reduces to the stable HiPPO-LegS SSM, while in other cases the relative spatial positions from other modes can be spread, compressed and memorized in the state space through the rank-one term.

\subsubsection{Cross-Mode Modulation}

Existing methods tend to ignore the correlation among the indices after learning a tensor core representation, which limits the quality of the tensor core generation. Our TensorSSM learns the representation \(\boldsymbol{y}_{S-1}^{(m)} \in \mathbb{R}^{H}\) for each mode index $i_m$, so beyond capturing the \textit{multi-scale spatial correlation in the state space}, the tensor core may also \textit{have correlation in the space of hidden states}. To further strengthen cross-mode interactions, we apply a cross-mode Feature-wise Linear Modulation (FiLM). For mode \(m\), we concatenate the representations of all other modes:
\begin{equation}
    \boldsymbol{c}^{(m)}
    =
    \bigoplus_{m' \neq m}
    \boldsymbol{y}_{S-1}^{(m')}
    \in
    \mathbb{R}^{H(M-1)}.
\end{equation}
This concatenated information from other modes serves as a general modulation on the generated representation vector $\boldsymbol{y}_{S-1}^{(m)}$ through a gate mechanism. The modulation factors $\boldsymbol{\xi}_m, \boldsymbol{\beta}_m \in \mathbb{R}^{H}$ can be derived by:
\begin{equation}
    \begin{bmatrix}
        \boldsymbol{\xi}_m \\
        \boldsymbol{\beta}_m
    \end{bmatrix}
    =
    \boldsymbol{W}_{c,2}^{(m)}
    \operatorname{SiLU}
    \left(
        \boldsymbol{W}_{c,1}^{(m)}
        \boldsymbol{c}^{(m)}
        +
        \boldsymbol{b}_{1}^{(m)}
    \right)
    +
    \boldsymbol{b}_{2}^{(m)},
\end{equation}
where $\boldsymbol{W}_{c,1}^{(m)} \in \mathbb{R}^{H \times H(M-1)}, \boldsymbol{W}_{c,2}^{(m)} \in \mathbb{R}^{2H \times H}$ are learnable matrices, $\boldsymbol{b}_{1}^{(m)}, \boldsymbol{b}_{2}^{(m)}$ are biases, and $\mathrm{SiLU}(\cdot)$ is the activation function. The modulated representation is
\begin{equation}
    \boldsymbol{r}^{(m)}
    =
    \operatorname{sigmoid}(\boldsymbol{\xi}_m)
    \odot
    \boldsymbol{y}_{S-1}^{(m)}
    +
    \boldsymbol{\beta}_m.
\end{equation}
This allows the representation of one mode to be conditionally rescaled by the global context of all other modes.

\subsection{Tucker Reconstruction and Residual Heads}
\label{sec:method-reconstruction}

Given the final INR of the $m$-th mode \(\boldsymbol{r}^{(m)} \in \mathbb{R}^{H}\), in this paper, we reconstruct the scalar tensor value using Tucker decomposition. It is noteworthy that tensor decomposition methods are flexible and the Tucker reconstruction head can also be replaced by CP, TT or MLP-based methods. Details of the relevant experiments are given in Section~\ref{reconstruction}.

For each mode \(m\), we project the representation into a low-dimensional factor:
\begin{equation}
    \boldsymbol{f}^{(m)}
    =
    \boldsymbol{W}_{\mathrm{tuck}}^{(m)}
    \boldsymbol{r}^{(m)}
    \in
    \mathbb{R}^{R_m},
\end{equation}
where \(\boldsymbol{W}_{\mathrm{tuck}}^{(m)} \in \mathbb{R}^{R_m \times H}\). Let $\boldsymbol{\mathcal{G}} \in \mathbb{R}^{R_1 \times R_2 \times \cdots \times R_M} $ be the Tucker core tensor. The Tucker prediction is
\begin{equation}
    z
    =
    \left\langle
        \boldsymbol{\mathcal{G}},
        \boldsymbol{f}^{(1)}
        \otimes
        \boldsymbol{f}^{(2)}
        \otimes
        \cdots
        \otimes
        \boldsymbol{f}^{(M)}
    \right\rangle.
\end{equation}
This head captures global high-order interactions among all tensor modes. Since the correlation of values and mode indices are highly complex, using the spatial-aware state space transition alone may suffer from the memory loss in the state space~\citep{wang2025memmamba}. Therefore, to further improve expressiveness, we add an optional pairwise residual head to perform a fine-grained refinement on the output using the cross-mode outputs $\boldsymbol{r}_p$ and $\boldsymbol{r}_q$, which deeply refines and recalls the cross-mode spatial information to enhance the final outputs. 
\begin{equation}
    \Delta z
    =
    \sum_{p < q}
    \left\langle
        \boldsymbol{W}_{L}^{pq}
        \boldsymbol{r}^{(p)},
        \boldsymbol{W}_{R}^{pq}
        \boldsymbol{r}^{(q)}
    \right\rangle,
\end{equation}
where \(\boldsymbol{W}_{L}^{pq}, \boldsymbol{W}_{R}^{pq} \in \mathbb{R}^{2\times H}\) are learnable projection matrices. Then the final prediction $\hat{y}$ is
\begin{equation}
    \hat{y}
    =
    \eta z
    +
    \kappa_0
    +
    \lambda \Delta z,
\end{equation}
where \(\eta\) and \(\kappa_0\) are scalar learnable parameters, and \(\lambda\) controls the strength of the pairwise residual.

\subsection{Training Objective}

The model is trained by minimizing the mean squared error over tensor entries:
\begin{equation}
    \mathcal{L}
    =
    \frac{1}{N_{\mathrm{total}}}
    \sum_{n=1}^{N_{\mathrm{total}}}
    \left(
        \hat{y}_n
        -
        y_n
    \right)^2,
\end{equation}
where \(y_n\) is the ground-truth value of the \(n\)-th tensor entry.

To prevent the model from relying on a fixed coordinate traversal order, we use permutation-based training. Following the previous work~\citep{kwon2023tensorcodec}, at each epoch, a random permutation over all tensor entries is sampled, and mini-batches are drawn from the permuted index sequence. This encourages the learned INR to depend only on the coordinate values and mode structure, not on the ordering in which entries are visited. During validation, the model is evaluated densely over the full tensor grid to obtain unbiased full-tensor MSE and fitness metrics.

%% file: Content/4Experiment.tex

\section{Experiments}
\label{sec:experiments}

\subsection{Datasets and Evaluation Metrics}
\label{sec:datasets_metrics}

We evaluate TensorSSM on six real-world tensors including urban mobility, human activity, environmental monitoring, and transportation. \textbf{Uber} is a $183 \times 24 \times 1{,}140$ tensor of pickup counts organized by date, hour, and latitude. \textbf{Action} is a $100 \times 570 \times 567$ tensor of frame-level features extracted from action videos. \textbf{Air Quality} is a $5{,}600 \times 362 \times 6$ tensor of pollutant measurements organized by hour, monitoring location, and pollutant type. \textbf{PEMS-SF} is a $963 \times 144 \times 440$ tensor of lane-usage rates organized by station, timestamp, and day. \textbf{Activity} is a $337 \times 570 \times 320$ tensor of frame-level features extracted from human-activity videos. \textbf{NYC} is a $265 \times 265 \times 28 \times 35$ tensor of taxi-trip counts organized by origin, destination, day, and month.

\begin{table}[t]
  \centering
  \caption{Statistics of the datasets used in our experiments.}
  \label{tab:dataset_statistics}

  \normalsize
  \setlength{\tabcolsep}{2pt}
  \renewcommand{\arraystretch}{1.12}

  \begin{tabular*}{\columnwidth}{
    @{\extracolsep{\fill}}
    lcccc
    @{}
  }
    \toprule
    \textbf{Dataset}
    & \textbf{Size}
    & \textbf{Order}
    & \textbf{Density}
    & \textbf{Smoothness} \\
    \midrule

    Uber
    & \(183{\times}24{\times}1{,}140\)
    & 3
    & 0.138
    & 0.861 \\

    Action
    & \(100{\times}570{\times}567\)
    & 3
    & 0.393
    & 0.484 \\

    Air
    & \(5{,}600{\times}362{\times}6\)
    & 3
    & 0.917
    & 0.513 \\

    PEMS-SF
    & \(963{\times}144{\times}440\)
    & 3
    & 0.999
    & 0.461 \\

    Activity
    & \(337{\times}570{\times}320\)
    & 3
    & 0.569
    & 0.553 \\

    NYC
    & \(265{\times}265{\times}28{\times}35\)
    & 4
    & 0.118
    & 0.788 \\

    \bottomrule
  \end{tabular*}
  \vspace{-3mm}
\end{table}

We assess compression performance using compressed size and reconstruction Fitness. Given an input tensor $\mathcal{X}$ and its reconstruction $\hat{\mathcal{X}}$, Fitness is defined as
\begin{equation}
    \operatorname{Fitness}
    = 1 - \frac{\lVert \mathcal{X} - \hat{\mathcal{X}} \rVert_F}
    {\lVert \mathcal{X} \rVert_F},
    \label{eq:fitness}
\end{equation}
where a higher value indicates a more accurate reconstruction and a value of one indicates exact reconstruction. Compressed size is reported in bytes, with lower values indicating more compact storage, and includes the stored representation and the metadata required for reconstruction. All metrics are computed over the complete reconstructed tensor; thus, this evaluation measures full-tensor reconstruction rather than generalization to held-out coordinates. For TensorSSM and the size-matched baselines, we select the available configuration whose compressed size is closest to that of TensorCodec and compare reconstruction Fitness. For SZ3, we instead select the operating point whose Fitness is closest to TensorCodec and compare their compressed sizes. 

\subsection{Baselines}
\label{sec:baselines}

We compare TensorSSM with six representative baselines covering classical
tensor decomposition, error-bounded scientific data compression, and neural
tensor compression methods.
\textbf{(1) CPD.}
Canonical Polyadic Decomposition approximates a tensor as a sum of rank-one
components.
\textbf{(2) TTD.}
Tensor-Train Decomposition represents a high-order tensor as a sequence of
compact low-rank cores.
\textbf{(3) TRD.}
Tensor-Ring Decomposition connects the decomposition cores in a cyclic
structure to provide a flexible low-rank representation.
\textbf{(4) SZ3.}
SZ3 is an error-bounded scientific data compressor for which we select an
operating point with Fitness closest to TensorCodec and compare their
compressed sizes.
\textbf{(5) NeuKron.}
NeuKron is a neural Kronecker-based compression method that recursively models
tensor entries using compact learned representations.
\textbf{(6) TensorCodec.}
TensorCodec is a neural tensor-train compression method that combines
mixed-radix folding, index reordering, and recurrent generation to reconstruct
tensor entries.

\subsection{Implementation Details}
\label{sec:implementation_details}

We implement TensorSSM in PyTorch and conduct all experiments on an NVIDIA GeForce RTX 5090 GPU. Model parameters and forward computation use FP32, while full-tensor squared errors are accumulated in FP64 for stable metric reporting. All TensorSSM configurations use Tucker rank $r=6$, two recurrent blocks, an expansion ratio of 2, decimal coordinate tokenization with base $b=10$, cross-mode FiLM, and a pairwise interaction head of rank 2. We set the dimension of hidden size and state $(H,N)=(24,5)$ for Uber, Action, and Air Quality, $(H,N)=(11,10)$ for PEMS-SF and Activity, and $(H,N)=(10,8)$ for NYC. We optimize all TensorSSM models using Adam with a learning rate of $2 \times 10^{-3}$, a batch size of 65,536, and a maximum of 1,000 epochs. At each epoch, all tensor entries are randomly permuted and visited once, after which we evaluate the complete tensor using batches of 131,072 entries and retain the checkpoint with the highest Fitness. Training stops when Fitness fails to improve by at least $10^{-4}$ for 30 consecutive epochs. All experiments use random seed 7. For the baselines, we control capacity using rank for CPD and TTD, a truncation threshold for TRD, an absolute error bound for SZ3, hidden width for NeuKron, and hidden size for TensorCodec. Since CPD, TTD, and TRD do not provide a direct mechanism for targeting an exact compressed size, we sweep their rank or truncation settings and report, for each dataset, the configuration whose compressed size is closest to that of TensorCodec. TensorSSM and NeuKron are selected analogously using their available hidden-size or hidden-width configurations. SZ3 is tuned separately to match the Fitness of TensorCodec, after which the two methods are compared by compressed size.

\subsection{Main Results}
\begin{table*}[t]
  \centering
  \caption{Comparison of compressed size and Fitness on six datasets, where CS denotes compressed size. For CPD, TTD, TRD, NeuKron, and TensorSSM, we report the available configuration whose compressed size is closest to that of TensorCodec. For SZ3, we report the operating point whose Fitness is closest to that of TensorCodec. The best results are in bold and the second best are underlined. Improve denotes the relative Fitness gain of TensorSSM over the strongest competing method on each dataset. ``$\ast$'' indicates that the improvements over baselines are statistically significant.}
  \label{tab:tensorssm_results}

  \renewcommand{\arraystretch}{1.1}

  \begin{tabular}{@{}llcccccccc@{}}
    \toprule
    Dataset
    & Metric
    & CPD
    & TTD
    & NeuKron
    & TRD
    & SZ3
    & TensorCodec
    & TensorSSM
    & Improve (\%) \\
    \midrule

    \multirow{2}{*}{Uber}
    & CS (bytes)
    & 156,368
    & \underline{153,468}
    & 155,154
    & 157,824
    & 646,537
    & 153,766
    & \textbf{152,720}
    & \multirow{2}{*}{2.55} \\
    & Fitness
    & 0.7289
    & 0.7355
    & 0.6925
    & 0.6409
    & \underline{0.7382}
    & 0.7380
    & \textbf{0.7570}$^{\ast}$
    & \\
    \midrule

    \multirow{2}{*}{Action}
    & CS (bytes)
    & 153,512
    & 167,264
    & \underline{153,388}
    & 153,760
    & 4,980,510
    & 155,558
    & \textbf{152,624}
    & \multirow{2}{*}{12.36} \\
    & Fitness
    & \underline{0.7725}
    & 0.6767
    & 0.3767
    & 0.7716
    & 0.7493
    & 0.7542
    & \textbf{0.8680}$^{\ast}$
    & \\
    \midrule

    \multirow{2}{*}{Air}
    & CS (bytes)
    & \textbf{143,256}
    & \underline{148,320}
    & 149,891
    & 158,416
    & 1,592,014
    & 160,999
    & 152,720
    & \multirow{2}{*}{7.34} \\
    & Fitness
    & 0.6775
    & \underline{0.6794}
    & 0.6580
    & 0.5880
    & 0.6761
    & 0.6762
    & \textbf{0.7293}$^{\ast}$
    & \\
    \midrule

    \multirow{2}{*}{PEMS-SF}
    & CS (bytes)
    & \underline{43,344}
    & \textbf{42,460}
    & 46,175
    & 45,472
    & 1,906,220
    & 46,323
    & 45,648
    & \multirow{2}{*}{0.40} \\
    & Fitness
    & \underline{0.6027}
    & 0.5898
    & 0.5547
    & 0.5616
    & 0.5991
    & 0.5991
    & \textbf{0.6051}$^{\ast}$
    & \\
    \midrule

    \multirow{2}{*}{Activity}
    & CS (bytes)
    & 49,120
    & \underline{46,992}
    & 48,201
    & 50,140
    & 2,125,643
    & 48,365
    & \textbf{45,648}
    & \multirow{2}{*}{10.77} \\
    & Fitness
    & 0.6256
    & 0.5488
    & 0.4291
    & \underline{0.6470}
    & 0.6310
    & 0.6310
    & \textbf{0.7167}$^{\ast}$
    & \\
    \midrule

    \multirow{2}{*}{NYC}
    & CS (bytes)
    & 47,520
    & 49,392
    & \textbf{47,071}
    & 49,536
    & 530,298
    & 48,067
    & \underline{47,408}
    & \multirow{2}{*}{8.05} \\
    & Fitness
    & 0.7488
    & 0.6933
    & 0.7086
    & 0.7010
    & \underline{0.7578}
    & 0.7568
    & \textbf{0.8188}$^{\ast}$
    & \\

    \bottomrule
  \end{tabular}
  \vspace{-5mm}
\end{table*}

We evaluate the reconstruction performance and
compression efficiency of TensorSSM against six representative baselines across six datasets. Specifically, we compare their reconstruction Fitness and compressed size, with the results reported in
Table~\ref{tab:tensorssm_results}. TensorSSM achieves the highest
Fitness on all six datasets, with an average relative improvement of $6.91\%$ over the strongest competing method on each dataset.
The largest gains are observed on Action ($12.36\%$) and Activity
($10.77\%$), followed by NYC ($8.05\%$) and Air Quality ($7.34\%$).
Smaller improvements are observed on Uber ($2.55\%$) and PEMS-SF
($0.40\%$), indicating that the magnitude of the gains varies
across datasets. Nevertheless, TensorSSM consistently achieves
the highest reconstruction Fitness across all evaluated datasets.

Classical tensor decomposition methods exhibit substantial performance variation across datasets. CPD achieves a Fitness of $0.7725$ on Action, while TTD performs best among classical methods on Air Quality dataset ($0.6794$), and TRD leads on Activity ($0.6470$). These differences reflect the sensitivity of low-rank decomposition methods to the underlying tensor structure. Although each factorization can efficiently represent certain patterns, its prescribed decomposition form constrains the interactions that can be captured under a limited parameter budget. In contrast, TensorSSM generates coordinate-dependent Tucker factors through dynamic state-space modeling, with additional cross-mode interactions to enrich the learned representation. This flexibility allows TensorSSM to achieve higher Fitness than all three classical methods across the evaluated datasets.

Among neural compression baselines, NeuKron shows relatively low Fitness on Action ($0.3767$) and Activity ($0.4291$). Its Kronecker-based representation, designed primarily for sparse and reorderable tensors, may be less suitable for the structures present in these datasets. TensorCodec provides stronger performance by using an LSTM to generate tensor-train cores from folded and reordered indices. However, TensorSSM improves its Fitness from $0.7542$ to $0.8680$ on Action and from $0.7568$ to $0.8188$ on NYC. Unlike TensorCodec, TensorSSM uses numeral-base coordinate encoding and dynamic state-space modeling, which is useful for memorizing spatial information and constructing coordinate-dependent representations. These differences provide a plausible explanation for the observed improvements, while their individual contributions are examined in the ablation study.

The improvements in reconstruction quality do not require larger compressed representations. TensorSSM achieves the smallest compressed size on Uber, Action, and Activity, and the second-smallest on NYC. On Air Quality and PEMS-SF, its compressed size remains close to those of the size-matched baselines while achieving higher Fitness. Compared with SZ3, TensorSSM also demonstrates favorable storage efficiency at the reported operating points. For example, on NYC, TensorSSM achieves a Fitness of $0.8188$ using $47{,}408$ bytes, whereas SZ3 achieves $0.7578$ using $530{,}298$ bytes. 

Overall, the results demonstrate that TensorSSM provides consistently strong reconstruction quality within a compact storage budget. Its coordinate-dependent core generation offers a flexible alternative to fixed low-rank factorizations and existing neural compression architectures. The contributions of its individual components are further investigated through ablation experiments.

\subsection{Convergence Analysis}
\label{sec:convergence_analysis}

\begin{figure*}
    \centering
    \includegraphics[width=0.9\textwidth]{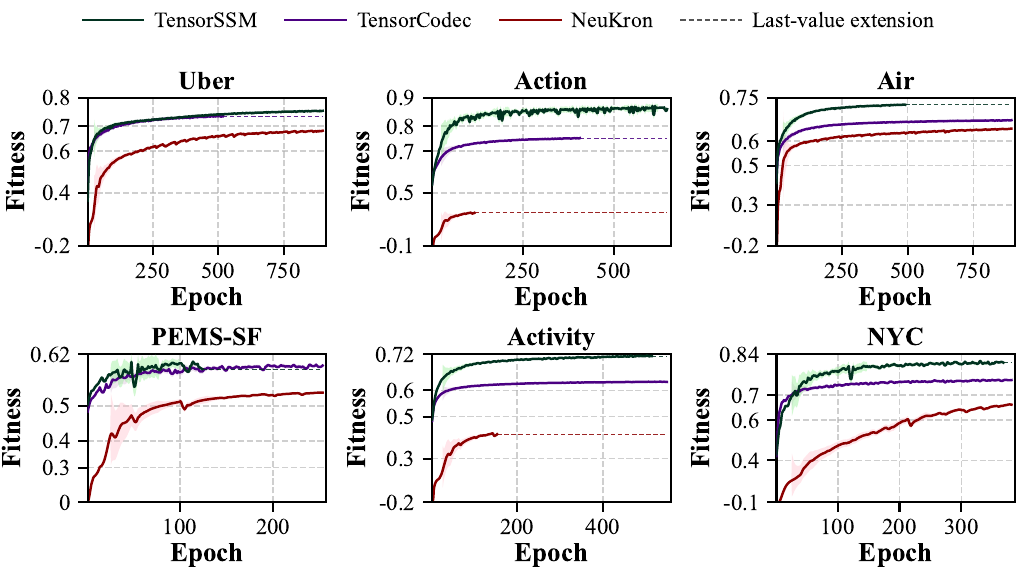}
    \caption{Training Fitness curves of TensorSSM, TensorCodec, and NeuKron on six datasets. For Uber, Air Quality, PEMS-SF, and NYC, the NeuKron curves are displayed up to the larger final epoch of TensorSSM and TensorCodec.}
    \label{fig:training_fitness_curves}
    \vspace{-3mm}
\end{figure*}

In this subsection, we investigate the convergence behavior of
TensorSSM by comparing its reconstruction Fitness with TensorCodec and NeuKron across six datasets. Figure~\ref{fig:training_fitness_curves}
presents the Fitness trajectories over training epochs. Dashed lines indicate last-value extensions for methods with shorter training runs and do not represent additional training results. Overall, TensorSSM achieves rapid improvements in the early training stages and maintains higher Fitness throughout most of the observed training process.

The convergence patterns vary across datasets. On Action and
Activity, TensorSSM establishes a clear advantage over both baselines within the early epochs and maintains this gap as training progresses. NeuKron converges to substantially lower Fitness on these datasets, while TensorCodec exhibits a smaller but persistent gap relative to TensorSSM. A similar pattern is observed on NYC, where TensorSSM reaches a higher Fitness level and remains ahead of the competing methods. 
On PEMS-SF, TensorSSM and TensorCodec exhibit comparable trajectories with some fluctuations, consistent with their smaller performance
difference.

These observations indicate that the reconstruction advantage of TensorSSM emerges early in training rather than relying solely on late-stage optimization. Its coarse-to-fine coordinate representation and cross-mode-enhanced state-space modeling may contribute to learning effective representations during training. Overall, TensorSSM demonstrates favorable convergence behavior in terms of training epochs across the evaluated datasets.

\subsection{Sensitivity to the Numeral Base}

In this subsection, we conduct a sensitivity analysis to investigate the effect of the numeral base on the reconstruction performance of TensorSSM. Specifically, a smaller base produces a longer coarse-to-fine sequence, whereas a larger base shortens the sequence but increases the token vocabulary size. The numeral base therefore controls the trade-off between sequence length and coordinate granularity. The results are presented in Figure~\ref{fig:numeral_base_sensitivity}.

\begin{figure}
    \centering
    \includegraphics[width=\linewidth]
    {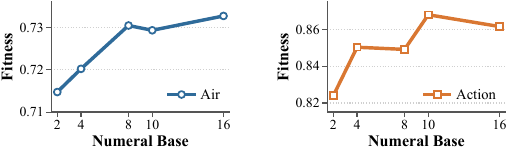}
    \caption{Effect of the numeral base on the Fitness of TensorSSM on the Air Quality and Action datasets.}
    \label{fig:numeral_base_sensitivity}
    \vspace{-3mm}
\end{figure}

Action shows a clear preference for decimal tokenization. Fitness increases from $0.8240$ at $b=2$ to $0.8680$ at $b=10$, and then decreases to $0.8616$ at $b=16$. The decimal setting improves Fitness by $0.0440$ over $b=2$ and by $0.0064$ over the next-best setting, $b=16$. This result indicates that $b=10$ provides a suitable balance for Action and supports its use as the default TensorSSM setting.

Air Quality dataset is less sensitive to the numeral base. Among the controlled settings, $b=8$, $b=10$, and $b=16$ achieve $0.7304$, $0.7293$, and $0.7327$, respectively. The best observed result at $b=16$ is only $0.0034$ above the decimal setting. This difference may be related to the imbalanced shape of the dataset, which is $(5600,362,6)$, since the dimension of size $6$ uses only a small subset of the available digits under larger bases. Thus, $b=10$ remains a strong default, although the optimal numeral base may vary with the tensor shape.

\subsection{Compatibility with Reconstruction Heads}
\label{reconstruction}
\begin{figure}
    \centering
    \includegraphics[width=\columnwidth]
    {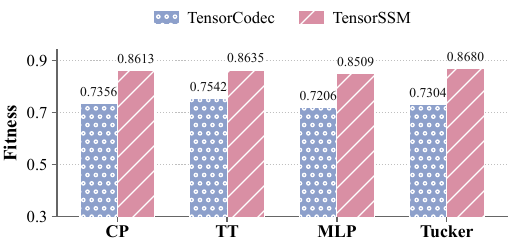}
    \caption{Fitness of TensorSSM and TensorCodec with different reconstruction heads on the Action dataset. 
    }
    \label{fig:tucker_replacement_ablation}
    \vspace{-4mm}
\end{figure}

In this subsection, we conduct a reconstruction-head replacement study to investigate whether the effectiveness of TensorSSM depends on the original Tucker reconstruction head. Specifically, we replace the Tucker head with CP, TT, and MLP alternatives. We also evaluate TensorCodec with the same head families, retaining its original TT head as the reference. Within each framework, the coordinate tokenizer, encoder, remaining modules, and training settings are kept unchanged, and each variant is trained end-to-end with seed $7$. The heads operate on mode-level representations in TensorSSM and mixed-radix-level representations in TensorCodec. The results are presented in Figure~\ref{fig:tucker_replacement_ablation}. The compressed sizes range from $154{,}110$ to $154{,}170$ bytes for TensorSSM and from $155{,}418$ to $155{,}566$ bytes for TensorCodec, providing approximately size-matched comparisons.

TensorSSM maintains a Fitness above $0.85$ with all three replacement heads. TT and CP obtain $0.8635$ and $0.8613$, respectively, only $0.0045$ and $0.0067$ below Tucker's $0.8680$. The MLP head achieves $0.8509$, indicating that TensorSSM can also be trained effectively with a generic neural reconstruction module, although the structured tensor heads provide better reconstruction quality. These results support the compatibility of TensorSSM with different reconstruction heads, with Tucker retained as the default because it gives the best observed performance.

TensorSSM also achieves higher Fitness than TensorCodec across all four head families. TensorCodec obtains $0.7356$, $0.7542$, $0.7206$, and $0.7304$ with CP, TT, MLP, and Tucker, respectively, with its original TT head performing best. The corresponding absolute Fitness gains of TensorSSM range from $0.1093$ to $0.1376$. TensorSSM therefore retains its reconstruction advantage across different head families, supporting the effectiveness of the framework beyond its default Tucker reconstruction head.

\subsection{Component Ablation}
\label{sec:component_ablation}

To evaluate the contribution of each component in TensorSSM, we conduct an ablation study on the Action dataset. Specifically, we individually remove the pairwise interaction module, the SSM module, and the coarse-to-fine tokenization mechanism. We also introduce two self-only variants for FiLM conditioning and feature fusion to examine the role of cross-mode information. Here, self-only means that each module uses only features from the current tensor mode, excluding information from other modes while retaining the module itself. The results are
reported in Table~\ref{tab:tensorssm_ablation_action}.

\begin{table}[t]
  \centering
  \caption{Ablation results of TensorSSM on the Action dataset.
  Self-only FiLM and self-only fusion remove information from the
  other tensor modes in the corresponding modules.}
  \label{tab:tensorssm_ablation_action}

  \normalsize
  \renewcommand{\arraystretch}{1.1}

  \begin{tabular*}{\columnwidth}{
    @{\extracolsep{\fill}}
    lr
    @{}
  }
    \toprule
    \textbf{Variant}
    & \textbf{Fitness} \\
    \midrule

    Full TensorSSM
    & \textbf{0.8680} \\

    Self-only FiLM
    & 0.8304 \\

    Self-only Fusion
    & 0.8390 \\

    w/o Pairwise Interaction
    & 0.8530 \\

    w/o SSM
    & 0.7778 \\

    w/o Coarse-to-Fine
    & 0.6819 \\

    \bottomrule
  \end{tabular*}
  \vspace{-3mm}
\end{table}
Table~\ref{tab:tensorssm_ablation_action} shows that the full TensorSSM achieves the highest Fitness, and every ablation causes a performance drop. The largest degradation occurs when coarse-to-fine tokenization is removed: Fitness decreases from $0.8680$ to $0.6819$. This variant retains the recurrent SSM but represents each raw coordinate in a single step, indicating that the hierarchical coordinate sequence is essential for learning informative multi-scale representations.

Removing the SSM reduces Fitness by $0.0902$ to $0.7778$, providing evidence that recurrent state propagation across coordinate scales is the key to tensor memory. The self-only FiLM and self-only fusion variants obtain $0.8304$ and $0.8390$, corresponding to drops of $0.0376$ and $0.0290$. Because these variants retain the corresponding modules but exclude information from the other tensor modes, their degradation directly demonstrates the value of cross-mode conditioning and feature exchange. With the FiLM module, the state memory can be enhanced by the cross-mode spatial information. Removing the pairwise interaction produces a smaller but consistent drop of $0.0150$, showing that explicit pairwise modeling provides a complementary improvement.

Overall, the results identify coarse-to-fine tokenization and recurrent state modeling as the two most influential components, while FiLM, cross-mode fusion, and pairwise interaction provide additional gains. The self-only variants have the same parameter count as the full model, so their differences are not caused by reduced capacity. In summary, the full model consistently outperforms all simplified variants, supporting the complementary design of its components.

\subsection{Comparison under Matched Hidden Size}

In this subsection, we conduct a hidden-size-matched comparison to investigate whether the performance advantage of TensorSSM over TensorCodec persists under the same hidden dimension. Specifically, we align the hidden representation sizes of both models across six datasets and compare their reconstruction Fitness. The results are presented in Figure~\ref{fig:same_width_fitness}. TensorSSM consistently outperforms TensorCodec on all six datasets, achieving an average absolute Fitness gain of approximately $0.114$. These results suggest that the improvement cannot be attributed solely to a larger hidden representation. One contributing factor may be the incorporation of cross-mode spatial information into the state-space representations, which enables TensorSSM to capture coordinate-dependent structures more effectively.

The largest improvements occur on NYC and Action, where Fitness increases from $0.634$ to $0.819$ and from $0.685$ to $0.868$, corresponding to absolute gains of $0.185$ and $0.183$. TensorSSM also improves Activity by $0.134$, Air Quality by $0.077$, and Uber by $0.066$. On PEMS-SF, the gain is smaller but remains positive, increasing Fitness from $0.564$ to $0.605$. The improvements across both three-way and four-way tensors show that TensorSSM can use the same nominal feature width more effectively across datasets with different orders and dimensions.

By controlling the hidden size, we can compare the reconstruction quality and directly validate the expressiveness of TensorSSM in memorizing tensor entries. These results also support the effectiveness of combining coarse-to-fine coordinate tokenization, which serves as the essential transformation of tensor entries, so that they can be learned and compressed within multi-scale sequences.

\begin{figure}
  \centering
  \includegraphics[width=\columnwidth]
  {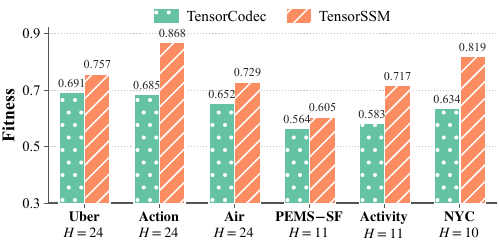}
  \caption{Fitness comparison between TensorSSM and TensorCodec under matched hidden sizes. The value of $H$ below each dataset denotes the hidden size shared by both models.}
  \label{fig:same_width_fitness}
  \vspace{-3mm}
\end{figure}

\subsection{Visualization of Spatial-Aware State Transition Operator}
\label{sec:spatial_transition_visualization}

In this subsection, we examine whether the cross-mode spatial information injected into the transition matrix actually changes how each mode memorizes its coordinate, and whether this change follows a consistent pattern. We use the model trained on Action without any further training. We fix the mode-1 coordinate at $i_1=50$ and sample 64 coordinates uniformly along each of mode 2 and mode 3, which gives a grid of 4,096 positions. We then inspect the transition matrices of the first SSM layer, which contains 24 channels with five states each, at the hundreds-, tens-, and units-digit scales. Specifically, the hundreds-digit scale controls the dynamics of tensor entries with mode indices from 1 to 999, then tens-digit scales controls the tensor entries with mode indices from 1 to 99 and units-digit scale controls indices from 1 to 9, which is a coarse-to-fine multi-scale paradigm. Because $i_1$ is fixed, the mode-1 coordinate tokens are identical at every grid position. Any change in the mode-1 transition matrix can therefore only come from the spatial information of mode 2 and mode 3.

\begin{figure}
    \centering
    \begin{subfigure}[t]{\linewidth}
        \centering
        \includegraphics[height=5.5cm]{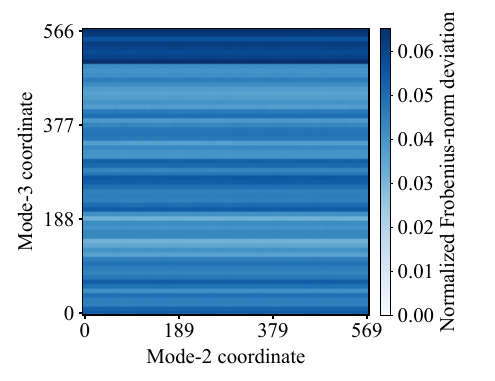} %
        \caption{Frobenius Norm of transition matrix in sampled entries.}
        \label{fig:spatial_A}
    \end{subfigure}
    \hfill
    \begin{subfigure}[t]{\linewidth}
        \centering
        \includegraphics[height=5.5cm]{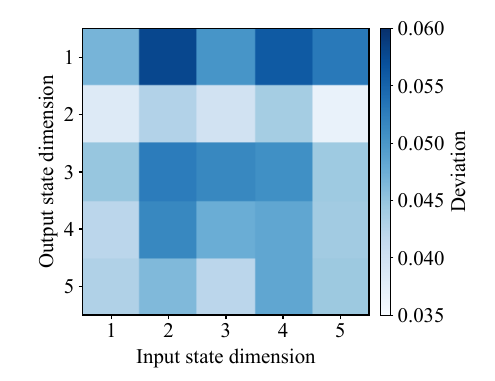}
        \caption{Entry-wise variation of transition operator.}
        \label{fig:spatial_B}
    \end{subfigure}
    \caption{Intensity of state transitions from the perspectives of entries and states. (a). Spatial modulation of the first-layer dynamic transition matrices on Action with $i_1=50$ fixed. The axes correspond to sampled mode-2 and mode-3 coordinates. (b). Entry-wise deviation of the first-layer dynamic transition matrices over the coordinate grid in (a). Rows and columns correspond to the two state dimensions, respectively.}
    \label{training}
    \vspace{-3mm}
\end{figure}

\begin{figure*}
 \centering
 \includegraphics[width=0.9\linewidth]{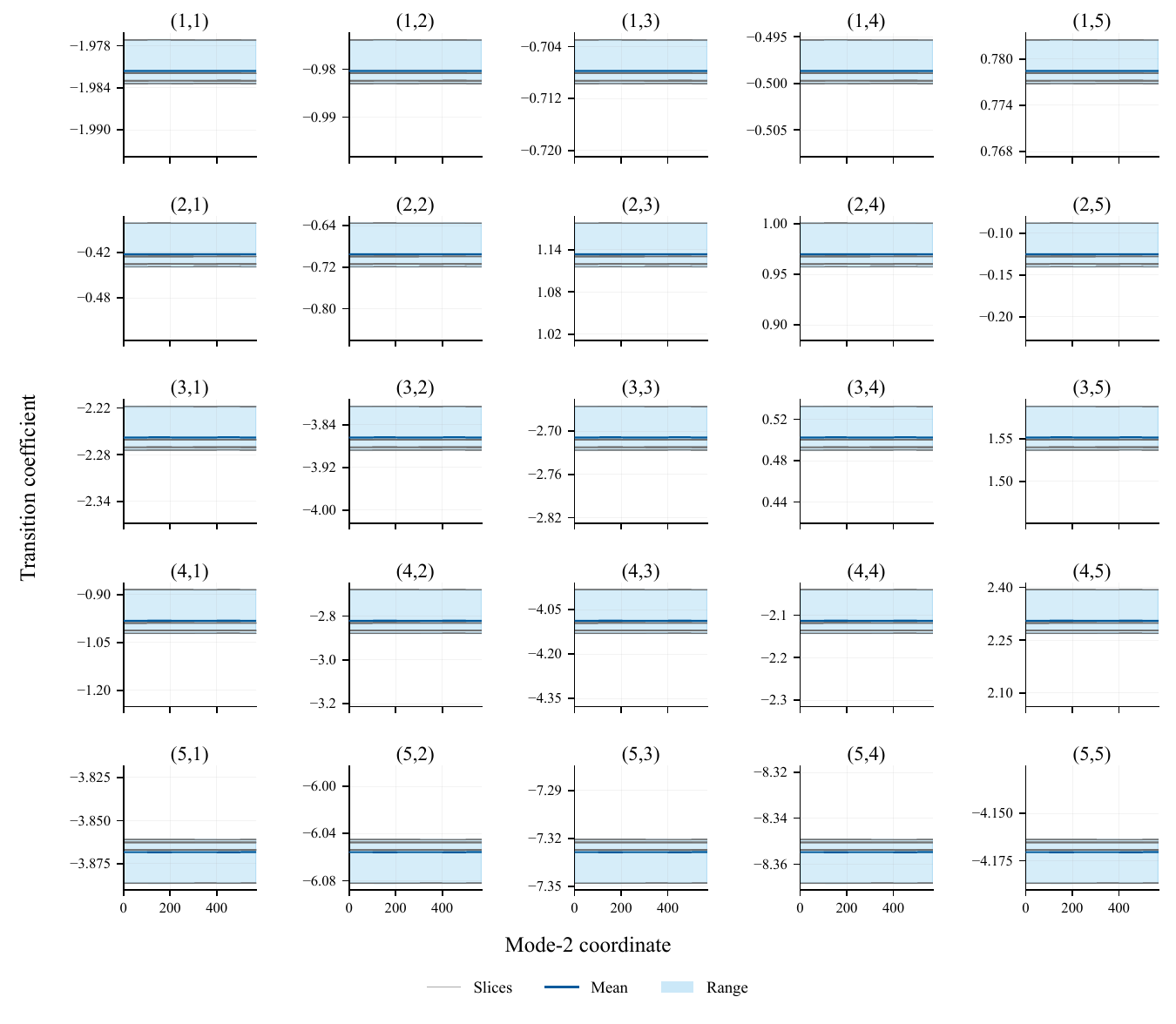}
  \caption{Raw transition coefficients at the hundreds-digit scale in the first SSM layer, with $i_1=50$. Panel $(r,c)$ corresponds to matrix row $r$ and column $c$. Each gray curve fixes one of 64 sampled mode-3 coordinates and varies over 64 sampled mode-2 coordinates. The blue line shows the slice mean, and the shaded region spans the slice minimum and maximum.}
  \label{fig:A_hundreds}
  \vspace{-5mm}
\end{figure*}

Figure~\ref{fig:spatial_A} shows that, with mode 1 fixed, the deviation of its transition matrix follows a clear directional pattern: it changes markedly with the mode-3 coordinate but hardly at all with the mode-2 coordinate. Each pixel corresponds to a position $(i_2, i_3)$, and its color measures how far the transition matrix at that position departs from the average over all positions. Moving along mode 3 for a fixed mode-2 coordinate, the color changes considerably and forms horizontal bands, with a maximum range of about $0.034$ within a single column. Different mode-3 positions thus modulate the mode-1 transition matrix with different strengths. In this sense, the spatial information of mode 3 induces distinguishable memory patterns, and the model has to adopt different patterns to absorb it. Moving along mode 2 for a fixed mode-3 coordinate, the color stays nearly constant, with a maximum range of only about $0.0008$ within a single row, roughly one fortieth of that along mode 3. The mode-2 position therefore has little effect on how mode 1 memorizes, and a nearly constant memory pattern suffices to handle the information from mode 2. The cross-mode modulation is thus selective: the model adjusts its memory pattern only for the modes whose spatial information is discriminative. 

\begin{figure*}
  \centering
  \includegraphics[width=0.9\linewidth]{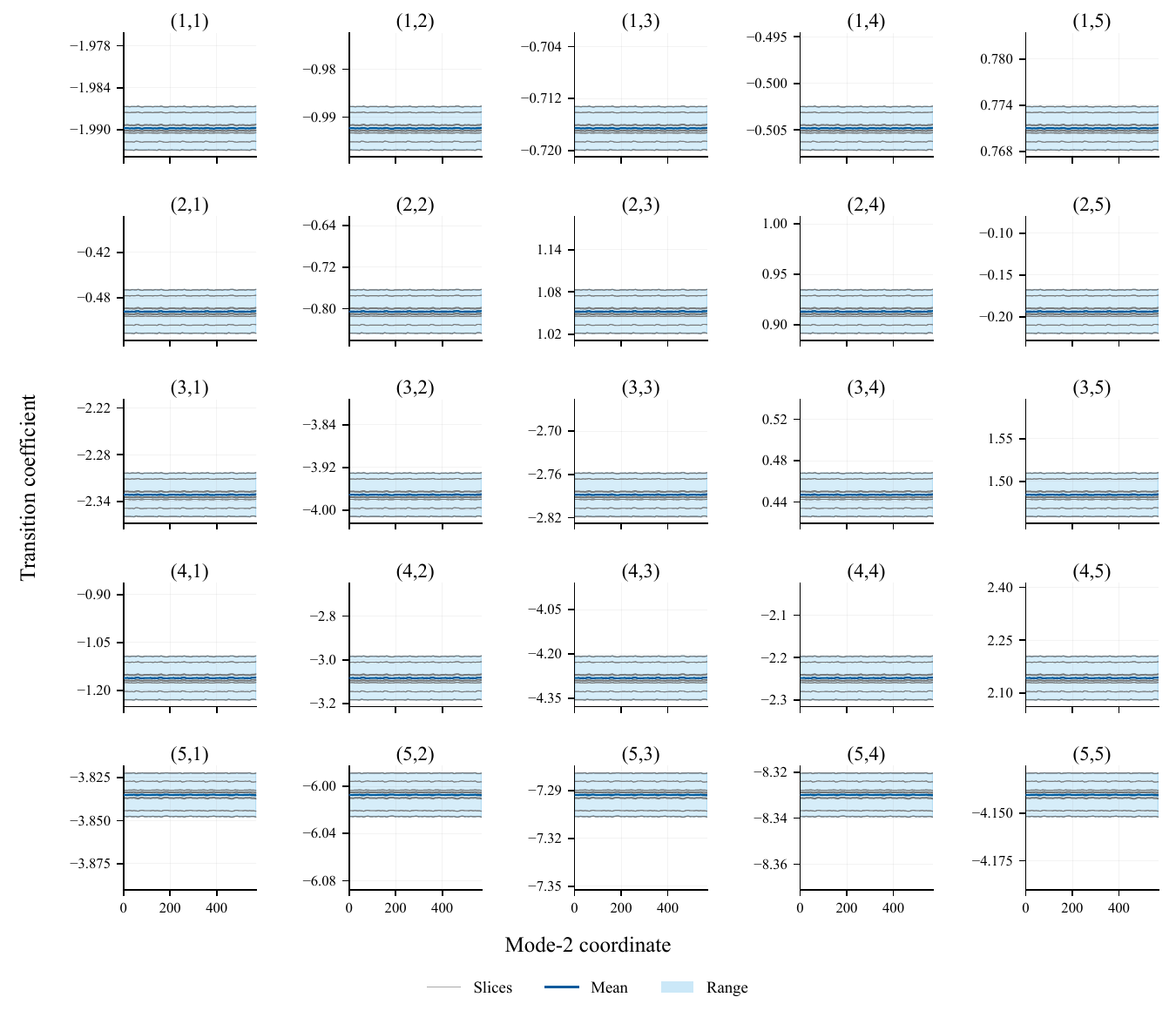}
  \caption{Raw transition coefficients at the tens-digit scale. Coordinates, layer, channel, curve definitions, and entry-wise vertical limits follow Figure~\ref{fig:A_hundreds}. Slices with identical same-scale digits at the same panel can produce overlapping curves.}
  \label{fig:A_tens}
  \vspace{-5mm}
\end{figure*}

Figure~\ref{fig:spatial_B} further shows that the spatial modulation does not alter each matrix entry independently, but adjusts the couplings between states in a structured way. Each cell reports how much the corresponding state values of mode 1 modulated by mode 2 and mode 3 vary over all 4,096 positions, which shows the memorization intensity deviation of the transition operator. The deviation $\Delta_A$ is calculated by: 
\begin{equation}
    \Delta_A = \sqrt{\dfrac{\sum_{i=1}^{N_1} \sum_{j=1}^{N_2} (\boldsymbol{\bar{A}}_{i,j}-\frac{1}{K} \sum_{s=1}^{K} \boldsymbol{\bar{A}}_{i,j}^{(s)})^2}{N_1 N_2}},
\end{equation}
where $\boldsymbol{\bar{A}}_{i,j}$ denotes the $i$-th row and $j$-th column of the discretized transition operator $\boldsymbol{\bar{A}}$, and $\boldsymbol{\bar{A}}^{(s)}$ denotes the $s$-th transition operator in the $K$ sampled entries. Entries in the same row vary by similar amounts, with within-row coefficients of variation between about 5.0\% and 7.5\%, whereas different rows differ clearly, e.g., about $0.053$ on average in the first row versus $0.040$ in the second. This structure follows from the dynamic correction $\Delta\boldsymbol{A}_{k}=\alpha u_{k}\boldsymbol{p}\boldsymbol{q}^{\top}$, where $\boldsymbol{q}$ determines which states are read, $\boldsymbol{p}$ determines which states the result is written back to, and the cross-mode input $u_{k}$ sets the strength of the modulation. The spatial information therefore acts on a coordinated set of state interactions, and the differences between rows reflect how $\boldsymbol{p}$ distributes the modulation across output states.

It can be observed that the variation among state values within each row of the transition operator is minimal. This phenomenon indicates that, when one mode is fixed, the other two modes (modes 2 and 3 in this instance) exert a highly homogeneous influence on the tensor core corresponding to the current mode (mode 1). This implies that fixing a specific mode induces consistent tensor-coordinate memory paradigms or numerical patterns.

Furthermore, for a fixed mode, a substantial discrepancy in deviations among the transition operators is also observed between the first and second rows of the transition operator: the first row exhibits consistently higher deviations (indicated by darker colors), whereas the second row remains predominantly lower (lighter colors). This demonstrates that, under a fixed mode, the memory influence paradigms of individual matrix entries exhibit structural commonalities along the row or column dimensions. Consequently, the update dynamics of a specific state can be enhanced at some specific state dimensions to memorize essential spatial correlations and value distributions.

Figure~\ref{fig:spatial_A} reports the magnitude of the variation averaged over channels and scales, but does not show what values the transition matrix actually takes or how they change with position. Figures~\ref{fig:A_hundreds}--\ref{fig:A_units} therefore plot the raw values of the transition-matrix entries. We take channel 4 as an example, whose overall variation is closest to the median of the 24 channels, as a representative, and show the results at the hundreds-, tens-, and units-digit scales separately. Each figure consists of $5\times5$ panels, and each panel corresponds to the entry value of transition operator $\boldsymbol{\bar{A}}$. Within a panel, each gray curve represents the dynamics at the specific position of $\boldsymbol{\bar{A}}$ among different sampled entries with fixed mode 1.

Three observations follow from these figures. First, every entry stays within a narrow range. For example, entry $(1,1)$ remains between $-1.990$ and $-1.978$ regardless of the mode-2 and mode-3 coordinates, a change of less than 1\%. The spatial information therefore only fine-tunes the base transition matrix without altering its overall structure, so the memory capability of the HiPPO matrix is preserved, consistent with the stability result of Theorem~1 in the appendix. Second, each gray curve is nearly flat, while different curves sit at different heights. The flat curves indicate that the entries barely change as the mode-2 coordinate varies, and the separation between curves indicates that changing the mode-3 coordinate moves the entries to different levels. This is the same directional pattern as in Figure~\ref{fig:spatial_A}, now visible directly in the coefficient values. Third, the same entry sits at different levels across the three scales; for instance, entry $(1,1)$ is generally lower at the tens-digit scale than at the hundreds-digit scale. The model thus uses different transition matrices when reading the hundreds, tens, and units digits, so coarse and fine positional information are memorized in different ways. In addition, many gray curves overlap because each scale depends only on the digit at that scale: the hundreds digit of mode 3 takes only six values (0--5), and the tens and units digits take ten values each, so the 64 curves can occupy at most six or ten distinct levels.

\begin{figure*}
  \centering
  \includegraphics[width=0.9\linewidth]
  {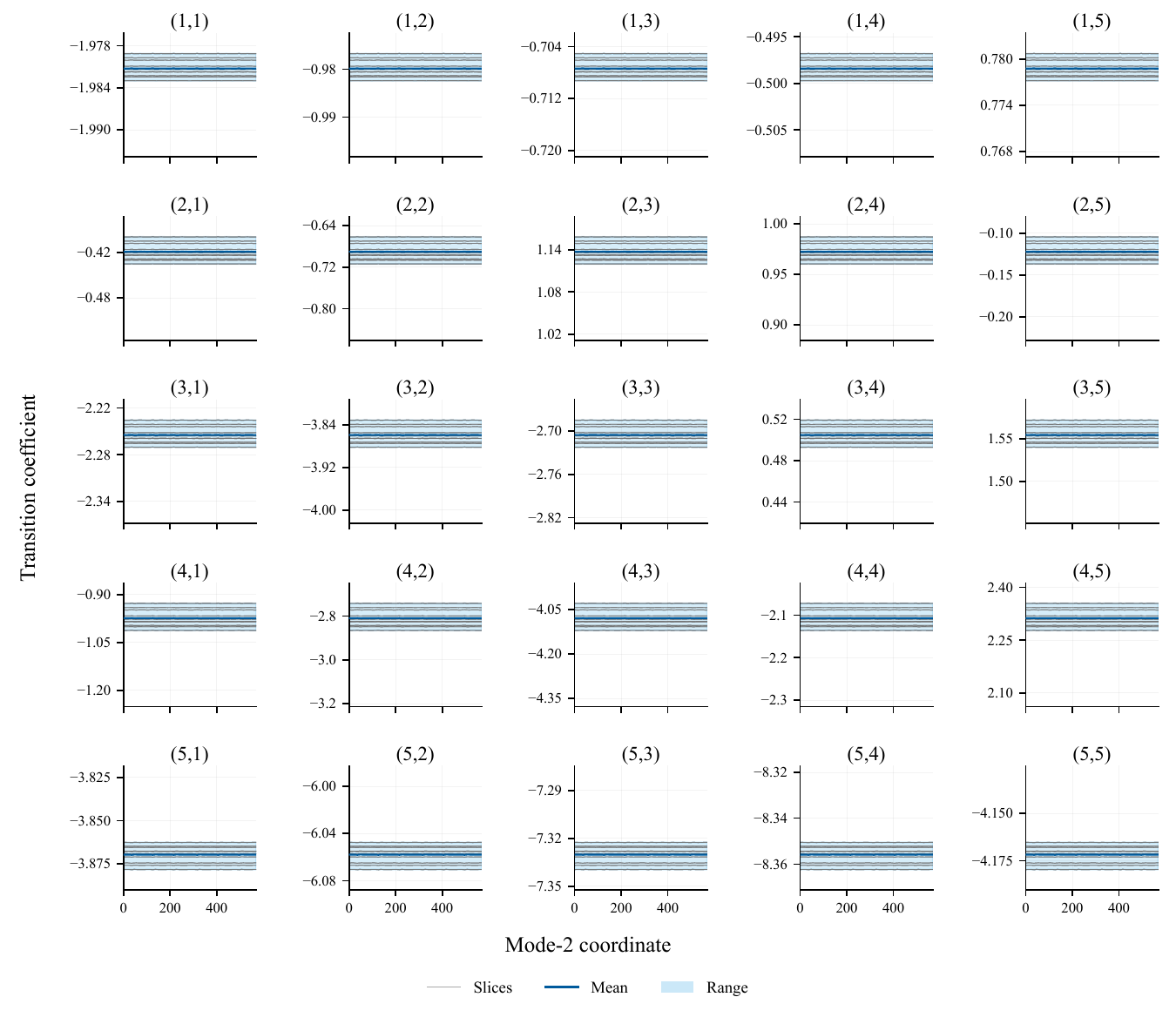}
  \caption{Raw transition coefficients at the units-digit scale. Coordinates, layer, channel, curve definitions, and entry-wise vertical limits follow Figure~\ref{fig:A_hundreds}. Slices with identical same-scale digits at the same panel can produce overlapping curves.}
  \label{fig:A_units}
  \vspace{-5mm}
\end{figure*}

\begin{figure}
    \centering
    \begin{subfigure}[t]{\linewidth}
        \centering
        \includegraphics[height=6.5cm]{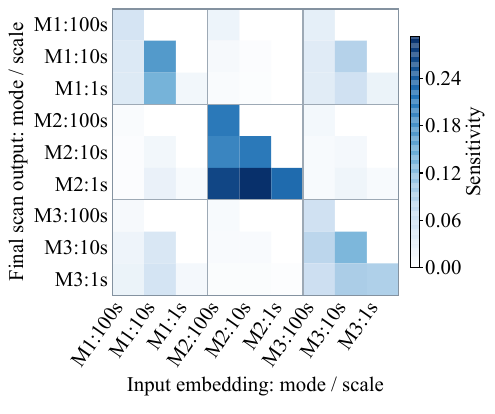}
        \caption{TensorSSM.}
    \end{subfigure}
    \begin{subfigure}[t]{\linewidth}
        \centering
        \includegraphics[height=6.5cm]{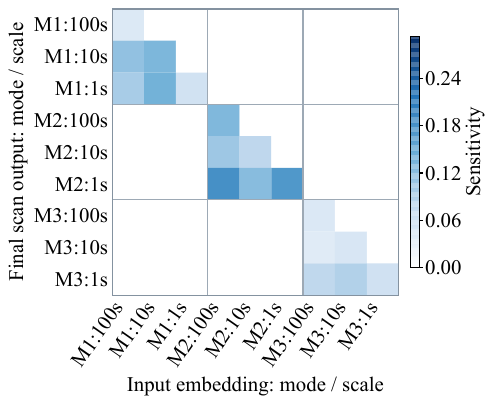}
        \caption{Self-only fusion.}
    \end{subfigure}
    \caption{Cross-mode sensitivity of the SSM outputs at the sampled entry $(94, 279, 330)$ of the Action dataset. Each cell reports the Frobenius norm $S_{a,b}$ of the channel Jacobian block between the input embedding at position $b$ (columns) and the final scan output at position $a$ (rows), where each position is indexed by mode (M1--M3) and digit scale (100s, 10s, 1s). Measurements are taken before the final channel mixing and FiLM, and both panels share the same color scale. (a). The full TensorSSM exhibits nonzero off-diagonal mode blocks, indicating that each SSM stream responds to the other modes. (b). In the self-only fusion variant, all off-diagonal mode blocks are exactly zero.}
    \label{fig:cross_mode_sensitivity}
    \vspace{-3mm}
\end{figure}

Figure~\ref{fig:cross_mode_sensitivity} complements the operator analysis by examining whether cross-mode dependencies are expressed in the SSM outputs. Within each mode, 100s, 10s, and 1s denote the hundreds-, tens-, and units-digit scales, showing the transition intensity in the multi-scale sequences. Each cell shows the Frobenius norm of the corresponding channel Jacobian block, including derivatives through the input-dependent transition parameters, which shows the sensitivity of the output at mode $a$ to the indices representations at mode $b$. This sensitivity reflects how the positions of the tensor entries affect the predicted values, which also reflects the memory efficacy of the INR. For a specific input position $b$ and output position $a$, we compute the derivatives across all channels to obtain a gradient matrix:
\begin{equation}
    J_{a,b}[h, k] = \frac{\partial y_{a,h}}{\partial e_{b,k}}.
\end{equation}
This represents the magnitude of change in the $h$-th channel of output position $a$ when the $k$-th channel of input position $b$ undergoes a slight perturbation. Then we calculate the Frobenius norm of the gradient matrix:
\begin{equation}
    S_{a,b} = \sqrt{\sum_{h=1}^{H} \sum_{k=1}^{H} J_{a,b}[h, k]^2},
\end{equation}
where $H$ is the number of channels. This value indicates the overall local sensitivity between the pair of positions. The measurements are taken before the final channel mixing and FiLM. Both panels share the same linear color scale. All reported tensor coordinates use zero-based indexing. At the sampled entry $(94,279,330)$ in the Action dataset, the full TensorSSM exhibits nonzero off-diagonal mode blocks, whereas these blocks are exactly zero in the self-only fusion variant at the measured stage. This contrast identifies the role of cross-mode feature exchange in making each SSM stream responsive to other modes before the final channel mixing and FiLM. The responses are also selective: for the mode-1 output block, sensitivity to mode-3 embeddings is larger than sensitivity to mode-2 embeddings, illustrating unequal cross-mode influence at this entry. Cross-mode blocks account for approximately 5.85\% of the total squared Jacobian norm, indicating that cross-mode sensitivity supplements the stronger within-mode dependencies rather than dominating the local response. The nonzero coarse-to-fine entries further show that smaller-scale embeddings can influence larger-scale outputs, while the zero reverse-direction entries reflect the causal scale ordering. Together, these patterns demonstrate the coexistence of cross-mode conditioning and propagation across coordinate scales.

In summary, these visualizations connect spatial conditioning at the operator level with cross-mode dependencies at the representation level. Other-mode coordinates modulate the transition operator even when the target coordinate is fixed; this modulation is structured by the low-rank state interaction and varies across coordinate scales. The sensitivity analysis further shows that cross-mode information influences the SSM outputs. These observations provide a mechanistic interpretation of how TensorSSM incorporates spatial context into state-space representations, complementing the reconstruction and ablation
results.

%% file: Content/5Conclusion.tex
\section{Conclusion}
In this paper, we propose TensorSSM, a novel Implicit Neural Representation for functional decomposition method that memorizes spatial positions within a dynamic State Space Model. 
To address the challenge of overlooking positional correlation among tensor entries, TensorSSM first decomposes mode-wise coordinates into sequences spit by K-ary rule, encoding spatial tokens from coarse to fine. Furthermore, the SSM's state space is refined to explicitly map cross-mode spatial dependencies, allowing inter-mode interactions to enhance the updates of each mode. By integrating cross-mode information in the state transition operator, TensorSSM effectively integrates spatial information into the state space and the tensor core, enhancing the expressiveness of the data compression. Extensive experiments demonstrate that TensorSSM achieves significantly high reconstruction accuracy and storage efficiency, exhibiting robustness and expressiveness across diverse tensor datasets.

%% file: Content/6Appendix.tex
\appendix

\section{Theoretical Analysis}

\subsection{Spectrum Properties of TensorSSM}
According to the previous work S4~\citep{gu2021efficiently}, if the eigenvalue modulus of the discretized transition operator $\bar{\boldsymbol{A}}$ is greater than 1, the increase in sequence length may lead to gradient explosion and inferior training performance.
If the modulus of the discretized transition operator $\bar{\boldsymbol{A}}$ is smaller than 1, which is the case for most random initialization or forced constraints for stable training, the model will exponentially forget information from in the long sequence. Therefore, in traditional SSM, it is difficult for ordinary $\boldsymbol{A}$ to maintain its eigenvalues around 1, so it is either unstable or unable to remember long-range information~\citep{nguyen2026muonssm}.

In this paper, we apply rank-one perturbation based on cross-mode spatial information to fine-tune the transition operator, enhancing the memorization of mode-specific tensor core. Since the length of multi-scale coarse-to-fine token sequence varies according to the predefined base-$b$ digits, in some settings, the lengths of token sequence may be also large. Therefore, we prove that the spatial-aware transition of SSM is also a stable update, which provides a solid support for TensorSSM for various sizes of tensors and tokenization settings.

\begin{theorem}[Dynamic $\boldsymbol{A}$ preserves Hurwitz property and multi-scale memory]
\label{thm:dynamic-A-eigen}
Let $\boldsymbol{A} \in \mathbb{R}^{N \times N}$ be the HiPPO-LegS matrix, with eigenvalues
$\lambda_n = -(n+1)$.  Let $\boldsymbol{P}, \boldsymbol{Q} \in \mathbb{R}^{N}$ be arbitrary vectors, $\alpha > 0$ a scalar, and
$u_k \in \mathbb{R}$ the input at step $k$.  Define the rank-one Dynamic-$\boldsymbol{A}$ perturbation
\begin{equation}
\label{eq:dynamic-Ak}
\boldsymbol{A}_k \;:=\; \boldsymbol{A} + \alpha \, u_k \, \boldsymbol{P} \boldsymbol{Q}^{\top} \;\in\; \mathbb{R}^{N \times N}.
\end{equation}
Let $\kappa(\boldsymbol{V})$ denote the condition number of the eigenvector matrix of $\boldsymbol{A}$
(see Lemma \ref{lem:hippo-diag} below).  Define the bilinear discretization
\begin{equation}\label{eq:bilinear-Ak}
\bar{\boldsymbol{A}}_k \;:=\; \Bigl(\boldsymbol{I} - \tfrac{\Delta t}{2} \boldsymbol{A}_k\Bigr)^{-1} \Bigl(\boldsymbol{I} + \tfrac{\Delta t}{2} \boldsymbol{A}_k\Bigr).
\end{equation}
Then there exists a constant $\delta_0 > 0$ that depends only on $\boldsymbol{A}$ such that
whenever
\begin{equation}\label{eq:dynamic-A-small}
\epsilon_k \;:=\; |\alpha| \, |u_k| \, \|\boldsymbol{P}\| \, \|\boldsymbol{Q}\| \;<\; \delta_0,
\end{equation}
the following hold:
\begin{enumerate}[leftmargin=*]
    \item[\normalfont(i)] (\textbf{$\boldsymbol{A}_k$ is Hurwitz.}) Every eigenvalue of $\boldsymbol{A}_k$ has strictly negative real part:
    \begin{equation}\label{eq:Ak-hurwitz}
    \operatorname{Re}\!\bigl(\lambda_n(\boldsymbol{A}_k)\bigr) \;<\; 0, \qquad n = 0, 1, \ldots, N-1.
    \end{equation}
    More precisely, for each $n$ there exists an integer $\sigma(n) \in \{0, 1, \ldots, N-1\}$ such that
    \begin{equation}\label{eq:eig-perturb}
    \bigl| \lambda_n(\boldsymbol{A}_k) - \lambda_{\sigma(n)}(\boldsymbol{A}) \bigr|
    \;\leq\; \kappa(\boldsymbol{V}) \cdot \epsilon_k.
    \end{equation}
    In particular, for sufficiently small $\epsilon_k$ the labeling is preserved ($\sigma(n) = n$).

    \item[\normalfont(ii)] (\textbf{$\bar{\boldsymbol{A}}_k$ is Schur stable.}) Every eigenvalue $\mu_n(\bar{\boldsymbol{A}}_k)$ of $\bar{\boldsymbol{A}}_k$ lies strictly inside the open unit disk:
    \begin{equation}\label{eq:Abar-k-schur}
    |\mu_n(\bar{\boldsymbol{A}}_k)| \;<\; 1, \qquad n = 0, 1, \ldots, N-1.
    \end{equation}

    \item[\normalfont(iii)] (\textbf{Multi-scale memory is preserved.}) For each $n$, define the $1/e$-decay time of the $n$-th mode of $\boldsymbol{A}_k$ by
    \begin{equation}
    \label{eq:tau-Ak}
    \tau_n^{(k)} \;:=\; -\frac{1}{\operatorname{Re}\!\bigl(\lambda_n(\boldsymbol{A}_k)\bigr)} \;>\; 0.
    \end{equation}
    Then
    \begin{equation}
    \label{eq:tau-perturb}
    \bigl| \tau_n^{(k)} - \tau_n \bigr|
    \;\leq\; C_{\tau} \cdot \kappa(\boldsymbol{V}) \cdot \epsilon_k,
    \end{equation}
    where $C_{\tau}$ is an absolute constant (independent of $\boldsymbol{A}_k$, $\boldsymbol{P}$, $\boldsymbol{Q}$, $\alpha$, $u_k$, $\Delta t$).
    Consequently the dynamic range satisfies
    \begin{equation}
    \label{eq:tau-range}
    \frac{\tau_0^{(k)}}{\tau_{N-1}^{(k)}} \;=\; N \cdot \bigl(1 + O(\kappa(\boldsymbol{V}) \epsilon_k)\bigr),
    \end{equation}
    i.e., the multi-scale memory structure within the dynamic $\boldsymbol{A}$ is preserved up to a $O(\kappa(\boldsymbol{V}) \epsilon_k)$ relative perturbation.
\end{enumerate}
\end{theorem}

\subsection{Proof}

\begin{proof}[Proof of Theorem \ref{thm:dynamic-A-eigen}]

We proceed in three steps.

\noindent\textbf{Diagonalization of $\boldsymbol{A}$ and Bauer-Fike setup}. First, we need to show that the HiPPO matrix $\boldsymbol{A}$ can be diagonalized, so that we can prove the fine-tuned spatial-aware state transition operator $\boldsymbol{\bar{A}}$ can also be diagonalized.

\begin{lemma}[Diagonalizability of the HiPPO-LegS matrix]\label{lem:hippo-diag}
The HiPPO-LegS matrix $\boldsymbol{A}$ is diagonalizable:
$\boldsymbol{A} = \boldsymbol{V} \boldsymbol{\Lambda} \boldsymbol{V}^{-1}$, where
$\boldsymbol{\Lambda} = \operatorname{diag}(\lambda_0, \lambda_1, \ldots, \lambda_{N-1})$
with $\lambda_n = -(n+1)$, and $\boldsymbol{V} \in \mathbb{R}^{N \times N}$ is the matrix of eigenvectors
with condition number $\kappa(\boldsymbol{V}) := \|\boldsymbol{V}\| \cdot \|\boldsymbol{V}^{-1}\| < \infty$.
\end{lemma}

\begin{proof}[Proof of Lemma \ref{lem:hippo-diag}]
The diagonalizability of HiPPO-LegS follows from the explicit construction of Theorem 2 in HiPPO~\citep{gu2020hippo}: the matrix $\boldsymbol{A}$ is the generator of a one-parameter family of orthogonal projections onto Legendre polynomial bases:
\begin{equation}
    \boldsymbol{A}_{nk} = -
    \begin{cases}
        (2n+1)^{1/2}(2k+1)^{1/2} & \textit{if } n > k, \\
        n + 1                    & \textit{if } n = k, \\
        0                        & \textit{if } n < k
    \end{cases},
\end{equation}
and this construction guarantees that $\boldsymbol{A}$ has $N$ linearly independent eigenvectors, and has $N$ distinct eigenvalues $\lambda_0, \ldots, \lambda_{N-1}$. Let $\boldsymbol{v}_n$ be any eigenvector of $\boldsymbol{A}$ for eigenvalue $\lambda_n$. Since the eigenvalues are distinct, the set $\{\boldsymbol{v}_0, \ldots, \boldsymbol{v}_{N-1}\}$ is linearly independent. Additionally, since a matrix is diagonalizable iff it has $N$ linearly independent eigenvectors, $\boldsymbol{A}$ is diagonalizable: $\boldsymbol{A} = \boldsymbol{V} \boldsymbol{\Lambda} \boldsymbol{V}^{-1}$ where $\boldsymbol{V} = [\boldsymbol{v}_0,  \cdots, \boldsymbol{v}_{N-1}]$ and $\boldsymbol{\Lambda} = \operatorname{diag}(\lambda_0, \ldots, \lambda_{N-1})$.
The condition number $\kappa(\boldsymbol{V}) = \|\boldsymbol{V}\| \|\boldsymbol{V}^{-1}\|$ is finite because $\boldsymbol{V}$ is a finite matrix with $\det(\boldsymbol{V}) \neq 0$, and it depends only on $N$.
\qedhere
\end{proof}

For convenience, we introduce the diagonalization
$\boldsymbol{A} = \boldsymbol{V} \boldsymbol{\Lambda} \boldsymbol{V}^{-1}$ with $\boldsymbol{\Lambda} = \operatorname{diag}(\lambda_0, \ldots, \lambda_{N-1})$.
Define
$\tilde{\boldsymbol{P}} := \boldsymbol{V}^{-1} \boldsymbol{P}$ and $\tilde{\boldsymbol{Q}} := \boldsymbol{V}^{\top} \boldsymbol{Q}$, so that $\boldsymbol{P} \boldsymbol{Q}^{\top} = \boldsymbol{V} \tilde{\boldsymbol{P}} \tilde{\boldsymbol{Q}}^{\top} \boldsymbol{V}^{-1}$.
The Dynamic-$\boldsymbol{A}$ matrix then becomes
\begin{equation}\label{eq:Ak-diag}
\boldsymbol{A}_k \;=\; \boldsymbol{V} \bigl( \boldsymbol{\Lambda} + \alpha u_k \tilde{\boldsymbol{P}} \tilde{\boldsymbol{Q}}^{\top} \bigr) \boldsymbol{V}^{-1}.
\end{equation}
Since $\boldsymbol{A}_k$ is similar to $\boldsymbol{\Lambda} + \alpha u_k \tilde{\boldsymbol{P}} \tilde{\boldsymbol{Q}}^{\top}$, the eigenvalues of
$\boldsymbol{A}_k$ are exactly the eigenvalues of the latter matrix, and $\kappa(\boldsymbol{V})$ is the relevant
condition number for eigenvalue perturbation.

\begin{lemma}[Bauer-Fike theorem]\label{lem:bauer-fike}
Let $\boldsymbol{M} = \boldsymbol{W} \boldsymbol{\Sigma} \boldsymbol{W}^{-1}$ be diagonalizable with condition number $\kappa(\boldsymbol{W}) = \|\boldsymbol{W}\| \|\boldsymbol{W}^{-1}\|$. Let $\boldsymbol{E}$ be any matrix. Then every eigenvalue $\tilde\lambda$ of $\boldsymbol{M} + \boldsymbol{E}$ satisfies
\begin{equation}\label{eq:bauer-fike}
\min_{j} |\tilde\lambda - \boldsymbol{\Sigma}_{jj}| \;\leq\; \kappa(\boldsymbol{W}) \cdot \|\boldsymbol{E}\|.
\end{equation}
\end{lemma}

\begin{proof}[Proof of Lemma \ref{lem:bauer-fike}]
This is the classical Bauer-Fike theorem~\citep{bauer1960norms}.

If $\tilde\lambda \in \{\boldsymbol{\Sigma}_{jj}\}$,
the left-hand side of \eqref{eq:bauer-fike} is $0 \le \kappa(\boldsymbol{V})\|\boldsymbol{E}\|$
(since $\kappa(\boldsymbol{V}) \ge 1$ for any normalized norm). Henceforth assume
$\tilde\lambda \notin \{\boldsymbol{\Sigma}_{jj}\}$, so $(\boldsymbol{\Sigma} - \tilde\lambda \boldsymbol{I})$
is invertible.

Since $\tilde\lambda$ is an eigenvalue of $\boldsymbol{M} + \boldsymbol{E}$, $\det(\boldsymbol{M} + \boldsymbol{E} - \tilde\lambda \boldsymbol{I}) = 0$. Rewriting via similarity, $\det(\boldsymbol{M} + \boldsymbol{E} - \tilde\lambda \boldsymbol{I}) = \det\bigl(\boldsymbol{V}^{-1}(\boldsymbol{M} + \boldsymbol{E} - \tilde\lambda \boldsymbol{I})\boldsymbol{V}\bigr) = \det(\boldsymbol{\Sigma} + \boldsymbol{V}^{-1}\boldsymbol{E}\boldsymbol{V} - \tilde\lambda \boldsymbol{I})$. Factoring out $(\boldsymbol{\Sigma} - \tilde\lambda \boldsymbol{I})$,
\begin{equation}
\label{eq:bauer-fike-det}
0 \;=\; \det(\boldsymbol{\Sigma} - \tilde\lambda \boldsymbol{I}) \cdot
    \det\!\bigl(\boldsymbol{I} + (\boldsymbol{\Sigma} - \tilde\lambda \boldsymbol{I})^{-1} \boldsymbol{V}^{-1} \boldsymbol{E} \boldsymbol{V}\bigr).
\end{equation}
Since $\det(\boldsymbol{\Sigma} - \tilde\lambda \boldsymbol{I}) \neq 0$, the second factor must vanish: $\det\bigl(\boldsymbol{I} + (\boldsymbol{\Sigma} - \tilde\lambda \boldsymbol{I})^{-1} \boldsymbol{V}^{-1} \boldsymbol{E} \boldsymbol{V}\bigr)
= 0$. Equivalently, $-1$ is an eigenvalue of
$(\boldsymbol{\Sigma} - \tilde\lambda \boldsymbol{I})^{-1} \boldsymbol{V}^{-1} \boldsymbol{E} \boldsymbol{V}$.

For any sub-multiplicative norm, $|\lambda| \le \|\boldsymbol{X}\|$ holds for every eigenvalue $\lambda$ of $\boldsymbol{X}$. Applying this to $\boldsymbol{A} = (\boldsymbol{\Sigma} - \tilde\lambda \boldsymbol{I})^{-1} \boldsymbol{V}^{-1} \boldsymbol{E} \boldsymbol{V}$ which has $-1$ as an eigenvalue and using property of $\|\boldsymbol{X}_1\boldsymbol{X}_2\boldsymbol{X}_3\| \le \|\boldsymbol{X}_1\| \|\boldsymbol{X}_2\| \|\boldsymbol{X}_3\|$,
\begin{equation}
\label{eq:bauer-fike-step3}
\begin{aligned}
    1 \;=\; |-1| &\;\le\;
    \bigl\|(\boldsymbol{\Sigma} - \tilde\lambda \boldsymbol{I})^{-1} \boldsymbol{V}^{-1} \boldsymbol{E} \boldsymbol{V}\bigr\|\\
    &\;\le\;
    \|(\boldsymbol{\Sigma} - \tilde\lambda \boldsymbol{I})^{-1}\| \cdot \|\boldsymbol{V}^{-1}\| \cdot \|\boldsymbol{V}\| \cdot \|\boldsymbol{E}\|.
\end{aligned}
\end{equation}

$(\boldsymbol{\Sigma} - \tilde\lambda \boldsymbol{I})^{-1}$ is diagonal with entries
$(\boldsymbol{\Sigma}_{jj} - \tilde\lambda)^{-1}$, so its operator norm equals the maximum absolute diagonal entry:
\begin{equation}
    \|(\boldsymbol{\Sigma} - \tilde\lambda \boldsymbol{I})^{-1}\| \;=\;
\max_j \frac{1}{|\boldsymbol{\Sigma}_{jj} - \tilde\lambda|}
\;=\; \frac{1}{\min_j |\tilde\lambda - \boldsymbol{\Sigma}_{jj}|}.
\end{equation}
Substituting this and $\kappa(\boldsymbol{V}) = \|\boldsymbol{V}\|\cdot\|\boldsymbol{V}^{-1}\|$ into
\eqref{eq:bauer-fike-step3} gives
\begin{equation}
    1 \;\le\; \frac{\kappa(\boldsymbol{V}) \cdot \|\boldsymbol{E}\|}{\min_j |\tilde\lambda - \boldsymbol{\Sigma}_{jj}|},
\end{equation}
which rearranges to the claimed bound.
\qedhere
\end{proof}

\noindent\textbf{Proof of (i).}
Then, we show that $\boldsymbol{{A}}_k$ is a Hurwitz matrix, thereby showing that in TensorSSM, spatial information integrated into the transition operator will maintain the basic property of SSM, so that it can be applied to  various sizes of tensors and tokenization settings.

\begin{proof}
Apply Lemma \ref{lem:bauer-fike} to $\boldsymbol{A} = \boldsymbol{V} \boldsymbol{\Lambda} \boldsymbol{V}^{-1}$ with the perturbation
$\boldsymbol{E} = \alpha u_k \boldsymbol{P} \boldsymbol{Q}^{\top}$.
Note that
$\|\boldsymbol{E}\| = |\alpha u_k| \cdot \|\boldsymbol{P}\| \cdot \|\boldsymbol{Q}\| = \epsilon_k$ (operator norm of a rank-one matrix $\boldsymbol{P} \boldsymbol{Q}^{\top}$ equals $\|\boldsymbol{P}\| \|\boldsymbol{Q}\|$). Then using Bauer-Fike theorem, we have:
\begin{equation}\label{eq:Ak-perturb}
\min_{j} \bigl| \lambda_n(\boldsymbol{A}_k) - \lambda_j(\boldsymbol{A}) \bigr|
\;\leq\; \kappa(\boldsymbol{V}) \cdot \epsilon_k, \qquad \forall n.
\end{equation}

By symmetry of the bound, this implies there exists a permutation $\sigma$ of
$\{0, 1, \ldots, N-1\}$ such that for every $n$,
\begin{equation}\label{eq:Ak-bounded}
\bigl| \lambda_n(\boldsymbol{A}_k) - \lambda_{\sigma(n)}(\boldsymbol{A}) \bigr| \;\leq\; \kappa(\boldsymbol{V}) \cdot \epsilon_k.
\end{equation}

Choose $\delta_0 = \min_{n \neq m} \frac{|\lambda_n(\boldsymbol{A}) - \lambda_m(\boldsymbol{A})|}{2 \kappa(\boldsymbol{V})}$.
Since $\lambda_n(\boldsymbol{A}) = -(n+1)$ are pairwise distinct with consecutive differences equal to $1$, we can take
\begin{equation}\label{eq:delta-0}
\delta_0 \;=\; \frac{1}{2 \kappa(\boldsymbol{V})} > 0.
\end{equation}

When $\epsilon_k < \delta_0$, the bound \eqref{eq:Ak-bounded} is strictly less than
$\frac{1}{2}$ of any consecutive $|\lambda_n - \lambda_m|$, so the permutation $\sigma$ is
forced to be the identity $\sigma(n) = n$. Therefore
\begin{equation}
\label{eq:Ak-perturb-id}
\begin{aligned}
    \bigl| \lambda_n(\boldsymbol{A}_k) - \lambda_n(\boldsymbol{A}) \bigr|
    &\;=\; \bigl| \lambda_n(\boldsymbol{A}_k) + (n+1) \bigr|\\
    &\;\leq\; \kappa(\boldsymbol{V}) \cdot \epsilon_k \;<\; \frac{1}{2}.
\end{aligned}
\end{equation}

Taking real parts, we derive the upper bound of $\operatorname{Re}\!\bigl(\lambda_n(\boldsymbol{A}_k)\bigr)$:
\begin{equation}
    \begin{aligned}
    \operatorname{Re}\!\bigl(\lambda_n(\boldsymbol{A}_k)\bigr)
    &\leq -(n+1) + \bigl| \lambda_n(\boldsymbol{A}_k) + (n+1) \bigr|\\
    &\leq -(n+1) + \kappa(\boldsymbol{V}) \cdot \epsilon_k \\
    &< -(n+1) + \dfrac{1}{2}. \label{eq:Ak-real-part}
    \end{aligned}
\end{equation}

For $n \geq 1$, $-(n+1) + \tfrac{1}{2} \leq -\tfrac{3}{2} < 0$, hence $\operatorname{Re}(\lambda_n(\boldsymbol{A}_k)) \leq -\tfrac{3}{2} < 0$. For $n = 0$, $-(n+1) + \tfrac{1}{2} = -\tfrac{1}{2} < 0$ as well. Therefore, all eigenvalues of $\boldsymbol{A}_k$ have strictly negative real parts, i.e., $\boldsymbol{A}_k$ is Hurwitz, which provides evidence for the stable linear system. This proves (i).
\qedhere
\end{proof}

\noindent\textbf{Proof of (ii)}. To show the stability of TensorSSM, we prove the eigenvalues of discretized $\boldsymbol{\bar{A}}$ is smaller than 1. Concretely, for any eigenvalue $\lambda$ with $\operatorname{Re}(\lambda) < 0$, the corresponding eigenvalue of $\bar{\boldsymbol{A}} = (\boldsymbol{I} - \tfrac{\Delta t}{2} \boldsymbol{A})^{-1}(\boldsymbol{I} + \tfrac{\Delta t}{2} \boldsymbol{A})$ is $\mu = (1 + \tfrac{\Delta t}{2} \lambda) / (1 - \tfrac{\Delta t}{2} \lambda)$ with $|\mu| < 1$.

\begin{proof}
Since $\boldsymbol{A}_k$ is Hurwitz by (i) and is similar to $\boldsymbol{\Lambda} + \alpha u_k \tilde{\boldsymbol{P}} \tilde{\boldsymbol{Q}}^{\top}$,
the same eigenvector computation in the diagonal basis of $\boldsymbol{\Lambda}$ applies to $\boldsymbol{A}_k$:
for any eigenpair $(\tilde{\boldsymbol{v}}_n, \lambda_n(\boldsymbol{A}_k))$,

\begin{equation}
    \begin{aligned}
    \bar{\boldsymbol{A}}_k \tilde{\boldsymbol{v}}_n
    &= \Bigl(\boldsymbol{I} - \tfrac{\Delta t}{2} \boldsymbol{A}_k\Bigr)^{-1} \Bigl(\boldsymbol{I} + \tfrac{\Delta t}{2} \boldsymbol{A}_k\Bigr) \tilde{\boldsymbol{v}}_n \notag \\
    &= \frac{1 + \tfrac{\Delta t}{2} \lambda_n(\boldsymbol{A}_k)}{1 - \tfrac{\Delta t}{2} \lambda_n(\boldsymbol{A}_k)} \, \tilde{\boldsymbol{v}}_n
    \;=:\; \mu_n(\bar{\boldsymbol{A}}_k) \, \tilde{\boldsymbol{v}}_n.
    \end{aligned}
\end{equation}

We assume $1 - \tfrac{\Delta t}{2} \lambda_n(\boldsymbol{A}_k) \neq 0$ (this holds because $\operatorname{Re}(\lambda_n(\boldsymbol{A}_k)) < 0$ by Hurwitz, so $\lambda_n(\boldsymbol{A}_k)$  is non-real-positive and $\Delta t \ge 0$).  Substituting into the definition $\bar{\boldsymbol{A}}_k = \bigl(\boldsymbol{I} - \tfrac{\Delta t}{2} \boldsymbol{A}_k\bigr)^{-1}
\bigl(\boldsymbol{I} + \tfrac{\Delta t}{2} \boldsymbol{A}_k\bigr)$, we expand term by term:
\begin{equation}
\begin{aligned}
\bar{\boldsymbol{A}}_k \tilde{\boldsymbol{v}}_n
&= \Bigl(\boldsymbol{I} - \tfrac{\Delta t}{2} \boldsymbol{A}_k\Bigr)^{-1}
   \Bigl(\boldsymbol{I} + \tfrac{\Delta t}{2} \boldsymbol{A}_k\Bigr) \tilde{\boldsymbol{v}}_n \\
&= \Bigl(\boldsymbol{I} - \tfrac{\Delta t}{2} \boldsymbol{A}_k\Bigr)^{-1}
   \bigl(\tilde{\boldsymbol{v}}_n + \tfrac{\Delta t}{2} \boldsymbol{A}_k \tilde{\boldsymbol{v}}_n\bigr) \\
&= \Bigl(\boldsymbol{I} - \tfrac{\Delta t}{2} \boldsymbol{A}_k\Bigr)^{-1}
   \bigl(\tilde{\boldsymbol{v}}_n + \tfrac{\Delta t}{2} \lambda_n(\boldsymbol{A}_k) \tilde{\boldsymbol{v}}_n\bigr) \\
&= \Bigl(\boldsymbol{I} - \tfrac{\Delta t}{2} \boldsymbol{A}_k\Bigr)^{-1}
   \Bigl(1 + \tfrac{\Delta t}{2} \lambda_n(\boldsymbol{A}_k)\Bigr) \tilde{\boldsymbol{v}}_n
   \\
&= \Bigl(1 + \tfrac{\Delta t}{2} \lambda_n(\boldsymbol{A}_k)\Bigr)
   \Bigl(\boldsymbol{I} - \tfrac{\Delta t}{2} \boldsymbol{A}_k\Bigr)^{-1} \tilde{\boldsymbol{v}}_n.
\end{aligned}
\end{equation}
Using the eigenpair equation again,
\begin{equation}
    \Bigl(\boldsymbol{I} - \tfrac{\Delta t}{2} \boldsymbol{A}_k\Bigr) \tilde{\boldsymbol{v}}_n
   = \tilde{\boldsymbol{v}}_n - \tfrac{\Delta t}{2} \lambda_n(\boldsymbol{A}_k) \tilde{\boldsymbol{v}}_n
   = \Bigl(1 - \tfrac{\Delta t}{2} \lambda_n(\boldsymbol{A}_k)\Bigr) \tilde{\boldsymbol{v}}_n.
\end{equation}
By the assumption $1 - \tfrac{\Delta t}{2} \lambda_n(\boldsymbol{A}_k) \neq 0$, we can invert:
\begin{equation}\label{eq:invert-v-n-A-k}
\Bigl(\boldsymbol{I} - \tfrac{\Delta t}{2} \boldsymbol{A}_k\Bigr)^{-1} \tilde{\boldsymbol{v}}_n
   \;=\; \frac{1}{1 - \frac{\Delta t}{2} \lambda_n(\boldsymbol{A}_k)} \, \tilde{\boldsymbol{v}}_n.
\end{equation}
Therefore, we derive the eigenvalue of $\boldsymbol{\bar{A}}$:
\begin{align}
\bar{\boldsymbol{A}}_k \tilde{\boldsymbol{v}}_n
&= \Bigl(1 + \tfrac{\Delta t}{2} \lambda_n(\boldsymbol{A}_k)\Bigr)
   \cdot \frac{1}{1 - \frac{\Delta t}{2} \lambda_n(\boldsymbol{A}_k)} \, \tilde{\boldsymbol{v}}_n \notag \\[3pt]
&= \frac{1 + \frac{\Delta t}{2} \lambda_n(\boldsymbol{A}_k)}{1 - \frac{\Delta t}{2} \lambda_n(\boldsymbol{A}_k)} \,
   \tilde{\boldsymbol{v}}_n
   \;=:\; \mu_n(\bar{\boldsymbol{A}}_k) \, \tilde{\boldsymbol{v}}_n. \notag
\end{align}

By (i), $\operatorname{Re}(\lambda_n(\boldsymbol{A}_k)) < 0$, so the bilinear map sends it to a point in
the open unit disk:
\begin{equation}\label{eq:Abar-k-modulus}
\bigl| \mu_n(\bar{\boldsymbol{A}}_k) \bigr|
\;=\; \left| \frac{1 + \tfrac{\Delta t}{2} \lambda_n(\boldsymbol{A}_k)}{1 - \tfrac{\Delta t}{2} \lambda_n(\boldsymbol{A}_k)} \right|
\;<\; 1.
\end{equation}
The strict inequality follows from the A-stability of the bilinear transform: the
mapping $\lambda \mapsto (1 + \tfrac{\Delta t}{2} \lambda)/(1 - \tfrac{\Delta t}{2} \lambda)$
is a Möbius transformation sending the open left half-plane bijectively onto the open
unit disk (and the imaginary axis onto the unit circle).  This proves (ii).
\qedhere
\end{proof}

\noindent\textbf{Proof of (iii)}.
Recall from the definitions of $\tau_n^{(k)}$ and $\tau_n$
\begin{equation}
\label{eq:tau-def-2}
\tau_n^{(k)} = -\frac{1}{\operatorname{Re}\!\bigl(\lambda_n(\boldsymbol{A}_k)\bigr)},
\qquad
\tau_n = -\frac{1}{\lambda_n(\boldsymbol{A})} = \frac{1}{n+1},
\end{equation}
with $\tau_n^{(k)}, \tau_n > 0$ by (i).


Applying the Lagrange’s Mean Value Theorem to the function $g(x) = -1/x$ on $(0, \infty)$, there exists $\xi \in (a,b)$ such that
\begin{equation}
\label{lagragian}
    g(b)-g(a) = g'(\xi)(b-a) = \frac{b-a}{\xi^{2}},
    \qquad g'(x) = \frac{1}{x^{2}}.
\end{equation}
Since $\xi > a$, we have $1/\xi^{2} < 1/a^{2}$, and hence
\begin{equation}
\label{lagragian_bound}
    \lvert g(b)-g(a)\rvert
    \;=\; \frac{\lvert a-b\rvert}{\xi^{2}}
    \;\leq\;
    \frac{\lvert a-b\rvert}{a^{2}}.
\end{equation}

Concretely, let $a = \tau_n = 1/(n+1)$ and $b = \tau_n^{(k)} = 1/(-\operatorname{Re}\lambda_n(\boldsymbol{A}_k))$.
From \eqref{eq:Ak-perturb-id},
\begin{equation}\label{eq:re-part-perturb}
\bigl| \operatorname{Re}\!\bigl(\lambda_n(\boldsymbol{A}_k)\bigr) - \lambda_n(\boldsymbol{A}) \bigr|
\;\leq\; \kappa(\boldsymbol{V}) \cdot \epsilon_k,
\end{equation}
so by applying the~\eqref{eq:re-part-perturb} to~\eqref{lagragian}, we have
\begin{equation}\label{eq:b-a-perturb}
|b - a| \;=\; \bigl| \tau_n^{(k)} - \tau_n \bigr| \;\leq\; \frac{\kappa(\boldsymbol{V}) \cdot \epsilon_k}{((n+1) - \kappa(\boldsymbol{V}) \epsilon_k)(n+1)}.
\end{equation}
For $\epsilon_k < \delta_0 = 1/(2\kappa(\boldsymbol{V}))$, the denominator
is bounded below by $\tfrac{1}{2}(n+1) \cdot (n+1) = \tfrac{(n+1)^2}{2}$.  Therefore,
\begin{equation}\label{eq:tau-perturb-2}
\bigl| \tau_n^{(k)} - \tau_n \bigr|
\;\leq\; \frac{2 \kappa(\boldsymbol{V}) \cdot \epsilon_k}{(n+1)^2}
\;\leq\; 2 \kappa(\boldsymbol{V}) \cdot \epsilon_k,
\end{equation}
since $(n+1)^2 \geq 1$.  This establishes \eqref{eq:tau-perturb} with
$C_{\tau} = 2$.

For the dynamic range \eqref{eq:tau-range}, write
\begin{align}
\frac{\tau_0^{(k)}}{\tau_{N-1}^{(k)}}
&= \frac{\tau_0 + (\tau_0^{(k)} - \tau_0)}{\tau_{N-1} + (\tau_{N-1}^{(k)} - \tau_{N-1})} \notag \\
&= \frac{\tau_0}{\tau_{N-1}} \cdot
\frac{1 + (\tau_0^{(k)} - \tau_0)/\tau_0}{1 + (\tau_{N-1}^{(k)} - \tau_{N-1})/\tau_{N-1}}.
\end{align}
Since $\tau_0 = 1$ and $\tau_{N-1} = 1/N$, the leading ratio is $N$. The fractional perturbations in numerator and denominator are each $O(\kappa(\boldsymbol{V}) \epsilon_k)$ by
\eqref{eq:tau-perturb-2}. A first-order expansion gives
\begin{equation}\label{eq:range-expand}
\frac{\tau_0^{(k)}}{\tau_{N-1}^{(k)}}
\;=\; N \cdot \bigl(1 + O(\kappa(\boldsymbol{V}) \epsilon_k)\bigr).
\end{equation}
This proves (iii). The input-dependent perturbation preserves the multi-scale memory structure of the HiPPO-LegS matrix: the memorization pattern at the dimension $n \in \{0, \ldots, N-1\}$ of fine-tuned transition operator  still has a characteristic decay timescale $\tau_n^{(k)} := -1/\operatorname{Re}\lambda_n(\boldsymbol{A}_k)$ within $O(\kappa(\boldsymbol{V})\epsilon_k)$ of the original $\tau_n = 1/(n+1)$, and the dynamic range $\max_n \tau_n^{(k)} / \min_n \tau_n^{(k)}$ stays $\Theta(N)$ rather than collapsing to $1$. This guarantees that the fine-tuned spatial-aware system compresses the input history at $N$ distinct temporal resolutions; no information is lost across the rank-one Dynamic-$\boldsymbol{A}$ update.
\qedhere
\end{proof}

\subsection{First-Order Refinement via Secular Equation}
\label{sec:secular-equation}

The Bauer-Fike bound \eqref{eq:Ak-perturb-id} is sharp in the worst case but can be sharpened when the rank-one structure of $\boldsymbol{E}$ is exploited.  The eigenvalues of $\boldsymbol{A}_k$ are the roots of the \textbf{secular equation}
\begin{equation}\label{eq:secular}
1 \;=\; \alpha u_k \sum_{n=0}^{N-1} \frac{\tilde{\boldsymbol{Q}}_n \tilde{\boldsymbol{P}}_n}{\lambda - \lambda_n},
\qquad \text{where } \lambda_n = -(n+1),
\end{equation}
obtained by applying the matrix determinant lemma to $\det(\lambda \boldsymbol{I} - \boldsymbol{A}_k) = 0$
with the diagonalization \eqref{eq:Ak-diag}.

For small $\alpha |u_k|$, expanding the secular equation near each $\lambda_n$ gives the
first-order eigenvalue perturbation
\begin{equation}\label{eq:lambda-1st-order}
\lambda_n(\boldsymbol{A}_k) \;=\; \lambda_n + \alpha u_k \, \tilde{\boldsymbol{Q}}_n \tilde{\boldsymbol{P}}_n + O\bigl((\alpha u_k)^2\bigr),
\end{equation}
where $\tilde{\boldsymbol{P}}_n, \tilde{\boldsymbol{Q}}_n$ are the $n$-th components of $\tilde{\boldsymbol{P}} = \boldsymbol{V}^{-1} \boldsymbol{P}$ and
$\tilde{\boldsymbol{Q}} = \boldsymbol{V}^{\top} \boldsymbol{Q}$.

\noindent \textbf{Interpretation of the first-order refinement}.
The coefficient $\tilde{\boldsymbol{Q}}_n \tilde{\boldsymbol{P}}_n$ is the projection of the rank-one perturbation onto the $n$-th eigenvector direction (in the diagonal basis of $\boldsymbol{A}$).  Different eigen-modes are shifted by different amounts depending on the alignment of $\boldsymbol{P}, \boldsymbol{Q}$ with the eigenvectors of $\boldsymbol{A}$.
In particular:
\begin{itemize}[leftmargin=*]
    \item If $\tilde{\boldsymbol{Q}}_n \tilde{\boldsymbol{P}}_n > 0$, the $n$-th mode becomes \emph{less negative}
    ($\lambda_n(\boldsymbol{A}_k)$ closer to zero), and the memory timescale $\tau_n^{(k)}$ \emph{increases}.
    \item If $\tilde{\boldsymbol{Q}}_n \tilde{\boldsymbol{P}}_n < 0$, the $n$-th mode becomes \emph{more negative}
    and the memory timescale \emph{decreases}.
    \item Modes with $\tilde{\boldsymbol{Q}}_n \tilde{\boldsymbol{P}}_n \approx 0$ are essentially unperturbed
    (corresponding to directions in which $\boldsymbol{P}$ and $\boldsymbol{Q}$ are orthogonal to the eigenvector).
\end{itemize}
The dynamic-$\boldsymbol{A}$ mechanism thus performs an \emph{input-dependent spatial-aware retuning of memory timescales}, and this is realized through a rank-one mechanism with cross-mode positional information compatible with HiPPO structure.